\documentclass[11pt]{article}

\title{Improved Regret for Efficient Online Reinforcement Learning \\
with Linear Function Approximation} 

\author{
Uri Sherman%
\thanks{Blavatnik School of Computer Science, Tel Aviv University; \texttt{urisherman@mail.tau.ac.il}.}
\and
Tomer Koren%
\thanks{Blavatnik School of Computer Science, Tel Aviv University and Google Research; \texttt{tkoren@tauex.tau.ac.il}.}
\and
Yishay Mansour%
\thanks{Blavatnik School of Computer Science, Tel Aviv University and Google Research; \texttt{mansour.yishay@gmail.com}.}
}

\usepackage{header_arxiv}
\usepackage{header}

\begin{document}

\maketitle

%%%%%%%%%%%%%%%%%%%
%%%% END ARXIV %%%%
%%%%%%%%%%%%%%%%%%%

\begin{abstract}
    We study reinforcement learning with linear function approximation and adversarially changing cost functions, a setup that has mostly been considered under simplifying assumptions such as full information feedback or exploratory conditions.
    We present a computationally efficient policy optimization algorithm for the challenging general setting of unknown dynamics and bandit feedback, featuring a combination of mirror-descent and least squares policy evaluation in an auxiliary MDP used to compute exploration bonuses.
    Our algorithm obtains an $\widetilde O(K^{6/7})$ regret bound, improving significantly over previous state-of-the-art of $\widetilde O (K^{14/15})$ in this setting.
    In addition, we present a version of the same algorithm under the assumption a simulator of the environment is available to the learner (but otherwise no exploratory assumptions are made), and prove it obtains state-of-the-art regret of $\widetilde O (K^{2/3})$.
\end{abstract}

\section{Introduction}
Reinforcement Learning (RL; \citealp{sutton2018reinforcement,mmt2022rlbook}) studies online decision making problems in which an agent learns through experience within a dynamic environment, with the goal to minimize a loss function associated with the agent-environment interaction.
Modern applications of RL such as robotics \cite{schulman2015high, lillicrap2015continuous, akkaya2019solving}, game playing \cite{mnih2013playing, silver2018general} and autonomous driving \cite{kiran2021deep}, almost invariably consist of large scale environments where function approximation techniques are necessary to allow the agent to generalize across different states. Furthermore, some form of agent robustness is usually required to cope with environment irregularities that cannot be faithfully represented by stochasticity assumptions \citep[see e.g.,][]{dulac2021challenges}.

Theoretical foundations for RL with function approximation \citep[e.g., ][]{jiang2017contextual,yang2019sample,jin2020provably,agarwal2020pc} 
have been steadily coming into fruition.
The influential work of \citet{jin2020provably} has set the ground for the de facto standard of linearly realizable RL; the linear Markov Decision Process (linear MDP), and has lead to a range of algorithmic approaches in this setting or variants thereof (e.g., \citealp{zanette2020frequentist, agarwal2020pc, wagenmaker2022reward}, see also \citealp{agarwal2019reinforcement}).
Likewise, a growing line of work studies RL with adversarial interventions, such as non-stationary dynamics \cite{mao2021near}, adversarial corruptions \cite{lykouris2021corruption}, delayed feedback \cite{lancewicki2022learning, jin2022near},
and adversarial costs \citep{even2009online,neu2012adversarial, rosenberg2019online,rosenberg2020stochastic, jin2020learning}. The latter is, arguably, the more fundamental and well studied setting in the scope of adversarial RL.

The present paper aims at advancing state-of-the-art algorithmic methods for computationally and statistically efficient RL in the linear MDP setup, under the challenging setting of \emph{adversarially changing costs, unknown dynamics, and bandit feedback}. 
At this time, there exist 
only a handful of papers that consider RL in a setup that combines function approximation and adversarial costs, with most prior works adopting one or more assumptions that alleviate the challenge of exploration. \citet{cai2020provably} was the first work to establish $\widetilde O(\sqrt K)$ regret over $K$ episodes in the related model of linear mixture MDP, yet considered full information feedback. Later, \citet{neu2021online} obtain the same minimax optimal rates in terms of $K$ for linear MDPs and bandit feedback, but with full knowledge of the environment dynamics, and an additional factor depending on the coverage of the initial state-action distribution. Finally, the recent work of \citet{luo2021policy} establishes an $\widetilde O(K^{14/15})$ guarantee in the linear MDP setup without any simplifying assumptions, and an $\widetilde O(K^{2/3})$ regret bound in the more general linear-$Q$ setting but with simulator access (albeit with a computationally inefficient algorithm). Notably, to the best of our knowledge, \citet{luo2021policy} is the only prior work to consider the  adversarial linear MDP with bandit feedback in its full generality.

\paragraph{Contributions.} Our main contribution significantly improves over the existing prior art \cite{luo2021policy} in a number of respects. We present a computationally efficient algorithm for the most general setup without any exploratory assumptions, and prove a regret bound of $\widetilde O(K^{6/7})$ establishing a substantial advancement with respect to the previous $\widetilde O(K^{14/15})$. 
In addition, we present a version of the same algorithm under the assumption a simulator is available to the learner, and prove it obtains an $\widetilde O(K^{2/3})$ bound matching the state-of-the-art in this setup given by the linear-$Q$ algorithm of \citet{luo2021policy} (which, notably, also applies in a more general setup). However, our algorithm improves upon that of \citet{luo2021policy} in being computationally efficient,
\footnote{In order to compute a single action probability of the agent policy, the algorithm of \citet{luo2021policy} requires exponentially many simulator samples, generated by traversing the tree structure implicitly defined by the recursive bonus-policy-bonus relation.} 
and in requiring a weaker simulator, which we use only to generate agent policy rollouts from the initial state. 
% By contrast, the algorithm of \citet{luo2021policy} requires a simulator that can generate next state samples from arbitrary state actions pairs, regardless of the agent policy's ability to reach those regions in the state space. 
Also noteworthy in this context is the algorithm of \citet{neu2021online}, which obtains an $\widetilde O (\sqrt K)$ regret bound, though requires not only a simulator but also perfect knowledge of the transition function. 
% Furthermore, the regret in \citet{neu2021online} also scales with the relative entropy between the benchmark policy state-action occupancy measure and the initial state-action distribution used in the learning process, thus in general can be much worse than $\sqrt K$.

\paragraph{Overview of techniques.}
Our work combines elements from \citet{jin2020provably,shani2020optimistic, neu2021online, luo2021policy} with a novel algorithmic approach towards exploration bonuses in linear MDPs.
We follow the insightful work of \citet{luo2021policy} and consider a regret decomposition and bonus design that at a high level are similar to those presented in their work, but reframed and extended to incorporate optimistic approximations of the \emph{bonus-to-go}; the bonus function that drives exploration.
Our central observation is that the bonus-to-go may be optimistically approximated using least squares regression
in the auxiliary full information \emph{bonus MDP}, in a manner that is efficient, and to an extent decoupled from estimation of the cost function.
The (non-linear) reward function in this MDP is the immediate bonus function that compensates for uncertainty in the instantaneous $Q$-estimates; importantly, while this is not a linear MDP, it is still amenable to least squares value backups \citep[e.g., ][]{jin2020provably} owed to the linear structure in the dynamics. 

During value backups in the bonus MDP, we incorporate an additional bonus in order to maintain (w.h.p.) Bellman consistency errors that are positive across the entire state action space. This is a form of optimism employed in policy optimization algorithms \citep[e.g., ][]{cai2020provably, shani2020optimistic}, where  the long term reward of the policy in each value backup step is overestimated (as opposed to optimizing a value function that is an overestimate of the reward of a benchmark policy). Unlike previous approaches that apply this directly towards the loss (or reward) optimization, here we utilize it solely for bonus calculation.
% This property is in fact the essential component that establishes the agent's exploration behavior counteracts the regret incurred from uncertainty in the estimates.
Finally, through a refined analysis, we simplify the framework of \citet{luo2021policy}, remove the necessity of the dilation component, and show we can use an immediate bonus function that is significantly smaller than that used in \citet{luo2021policy}. In particular, we keep the immediate bonus bounded (almost surely) by a constant across the entire state-action space, a property that is essential to arrive at a tighter bound for the least squares estimation procedure. 

\subsection{Additional Related Work}
\paragraph{Tabular RL with stationary and adversarial losses.}
Tabular RL with stationary losses is perhaps the most fundamental and well studied framework, beginning with the works of \citet{auer2006logarithmic,tewari2007optimistic,jaksch2010near}, and with many important advances more recently \cite{dann2015sample,azar2017minimax, dann2017unifying,fruit2018efficient,jin2018q}.
In the context of policy optimization methods in particular, most of the recent works consider the pure optimization perspective or under simplifying exploratory assumptions \citep[e.g., ][]{bhandari2019global,agarwal2021theory,zhan2021policy,lan2022policy}, with the exception of \citet{shani2020optimistic} that study the exploration setting and will be discussed momentarily.

The study of adversarially changing costs was initiated in the works of \citet{even2009online,yu2009markov}, and can be largely divided into policy optimization (PO) based methods \cite{neu2010online,shani2020optimistic} and algorithms that optimize over the set of occupancy measures \cite{zimin2013online,rosenberg2019online,jin2020learning}, where both approaches
ultimately involve a mirror descent \cite{nemirovskij1983problem,beck2003mirror} optimization component with online guarantees.
In the context of PO methods, which are more relevant to our work, \citet{neu2010online} initially achieve $\widetilde O(K^{2/3})$ regret for the known dynamics setup with bandit feedback.
In a later paper, \citet{shani2020optimistic} present PO algorithms based on value backups for the stochastic and adversarial settings with \emph{unknown} dynamics and bandit feedback, establishing an $\widetilde O(\sqrt K)$ bound in the stochastic case and $\widetilde O(K^{2/3})$ in the adversarial case.
The recent work of \citet{luo2021policy} presents, for the tabular case, a PO algorithm and prove it obtains the optimal $\widetilde O(\sqrt K)$ bound. Their algorithm, as opposed to that of \citet{shani2020optimistic}, is not based on value backups but rather stochastic estimates of the cumulative cost. The algorithm we present here combines both approaches.

\paragraph{RL with function approximation.}
The study of function approximation in RL goes back a long way (e.g., \citealp{schweitzer1985generalized,barto1990connectionist,bradtke1996linear}; see also \citealp{sutton2018reinforcement} and references therein), although these earlier works did not provide polynomial sample efficiency.
More recently, a line of work initiated by \citet{yang2019sample,yang2020reinforcement,jin2020provably}, studies MDPs with linear structure and focuses on computationally and statistically efficient algorithms \citep[e.g., ][]{zanette2020learning,modi2020sample,wei2021learning,wagenmaker2022first}.
The linear MDP model we adopt here was introduced by \citet{jin2020provably}.
Also noteworthy is the linear mixture MDP \cite{modi2020sample,ayoub2020model,zhou2021nearly,zhou2021provably}, which is a different model that in general is incomparable with the linear MDP \cite{zhou2021provably}.
Finally, there is a rich line of works studying statistical properties of RL with more general function approximation \citep[e.g., ][]{munos2005error,jiang2017contextual,dong2020root,jin2021bellman,du2021bilinear}, although these usually do not provide computationally efficient algorithms.

\paragraph{Policy optimization with function approximation.}
Formulation of policy optimization methods that incorporate function approximation was given in classical works such as \citet{sutton1999policy,kakade2001natural}, although these did not study convergence rates nor learning in the exploration setting. More recently,
several papers \citep[e.g., ][]{agarwal2021theory,liu2019neural} consider convergence properties of policy optimization approaches from a pure optimization perspective, or subject to exploratory assumptions such as bounded concentrability coefficient \cite{munos2003error,munos2005error,chen2019information},  distribution mismatch coefficient or a relative eigenvalue condition \cite{agarwal2021theory}.
More relevant to our paper are works that consider policy optimization in a setup that requires exploration be handled algorithmically, such as \citet{zanette2021cautiously} who improve upon the prior work of  \citet{agarwal2020pc}, both of which consider stationary losses.
The work of \citet{cai2020provably} that was mentioned earlier studies the adversarial setting, but in the linear mixture MDP model and with full information feedback. The recent work of \citet{he2022near} considers a similar setup and improves upon \citet{cai2020provably} by establishing better dependence on the horizon. 

\section{Problem Setup}

\paragraph{Episodic MDPs.}
A finite horizon episodic MDP is defined by the tuple
$\M = \br{\S , \A, H, \P , \l, s_1}$,
where $\S$ denotes the state space, $\A$ the action set, $H \in \mathbb Z_+$ the length of the horizon,
$\P = \cb{\P_h}_{h=1}^{H-1}$ the time dependent transition function, $\l = \cb{\l_h}_{h=1}^H$ a sequence of loss functions, and $s_1 \in \S$ the initial state that we assume to be fixed w.l.o.g.
The transition density given the agent is at state $s\in \S$ at time $h$ and takes action $a$ is given by $\P_h(\cdot | s, a) \in \Delta(\S)$. After the agent takes an action on the last time step $H$, the episode terminates immediately.
We assume the state space $\S$ is measurable space (which may contain uncountably many states) and the action set $\A$ is finite with $A \eqq |\A|$. 
A policy is defined by a mapping $\pi\colon \S \times [H] \to \Delta(\A)$, where $\Delta(\A)$ denotes the probability simplex over the action set $\A$. We let $\pi_h(\cdot|s) \in \Delta(\A)$ denote the distribution over actions given by $\pi$ at $s, h$.
Finally, we use the convention that for any function $V \colon \S \to \R$, we interpret $\P_h V \colon \S \times \A \to \R$ as the result of applying the conditional expectation operator $\P_h$; $\P_h V(s, a) \eqq \E_{s' \sim \P_h(\cdot |s, a)} V(s')$ (see \cref{sec:appdx_prelim} for comments regarding this notation). 
 
\paragraph{Episodic Linear MDPs with adversarial costs.} We consider the adversarial online learning setup, with unknown dynamics and bandit feedback. In this setup, the agent interacts with the MDP over the course of $K \geq 1$ episodes, 
where in each episode, the loss function associated with the MDP changes as chosen by an adversary that observes the current and past player policies. The feedback provided to the learner consists of the instantaneous scalar loss associated with the state-action pairs she has visited during episode rollout.
Our central structural assumption is that the combination of transition function and adversarial losses form a \emph{linear MDP} \cite{jin2020provably} in each episode. 
\begin{assumption}[Linear MDP with changing costs]
\label{assume:linmdp}
\label{def:linmdp}
    The learner interacts with a sequence of MDPs $\cb[b]{\M^k}_{k=1}^K$, $\M^k = (\S, \A, H, \P, \l^k, s_1)$ that share all elements other than the loss functions, such that the following holds.
    There is a feature mapping $\phi\colon \S \times \A \to \R^d$ that is \textbf{known} to the learner, and for every $h\in[H]$, $d$ \textbf{unknown} signed measures $\psi_{h, 1}, \ldots, \psi_{h, d} \in \S \to \R$ forming $\psi_h(\cdot) \eqq \br{\psi_{h, 1}(\cdot), \ldots, \psi_{h, d}(\cdot)}\in \S \to \R^d$,
    such that for all $h, s, a, s' \in [H-1] \times \S \times \A \times \S$:
    \begin{align}
		\P_h(s' | s, a)	&= \phi(s, a)\T \psi_h(s').
    \end{align}
    W.l.o.g., we assume $\norm{\phi(s, a)}\leq 1$ for all $s, a$,
    and that for any measurable function $f\colon \S \to \R$ with $\norm{f}_\infty \leq 1$, it holds that $\norm{\int \psi_h(s') f(s') {\rm d} s'} \leq \sqrt d$ for all $h\in [H]$.
    In addition, for all $k$;
    \begin{align}
	\l_h^k(s, a) &= \phi(s, a)\T \vc_h^k,
    \end{align}
    where $\cb{\vc_h^k}$ are adversarially chosen cost vectors. W.l.o.g., we assume $\av{\phi(s, a)\T \vc_h^k} \leq 1$ for all $s, a, h, k$, and
	$\norm{\vc_h^k} \leq \sqrt d$ for all $h, k$.
\end{assumption}

% \begin{definition}[Linear MDP \cite{jin2020provably}]
% \label{def:linmdp}
% 	An MDP $\M = \br{\S , \A, H, \P , \l, s_1}$ is \textbf{linear} if there exist a \textbf{known} feature mapping $\phi\colon \S \times \A \to \R^d$, For all $h\in[H]$, $d$ \textbf{unknown} signed measures $\psi_{h, 1}, \ldots, \psi_{h, d} \in \S \to \R$ forming $\psi_h(\cdot) \eqq \br{\psi_{h, 1}(\cdot), \ldots, \psi_{h, d}(\cdot)}\in \S \to \R^d$, and $H$ \textbf{unknown} cost vectors $\cb{\vc_h} \subset \R^d$ such that for all $s, a, s' \in \S \times \A \times \S$:
% 	\begin{align}
% 		\P_h(s' | s, a)	&= \phi(s, a)\T \psi_h(s'),
% 	\\ 
% 	\l_h(s, a) &= \phi(s, a)\T \vc_h.
% 	\end{align}
% 	Without loss of generality, we assume $\norm{\phi(s, a)}\leq 1$ and $\av{\phi(s, a)\T \vc_h} \leq 1$ for all $s, a, h \in \S \times \A \times [H]$, and
% 	$\norm{\vc_h} \leq \sqrt d$ for all $h \in [H]$. Finally, for any measurable function $f\colon \S \to \R$ with $\norm{f}_\infty \leq 1$, we assume $\norm{\int \psi_h(s') f(s') {\rm d} s'} \leq \sqrt d$ for all $h\in [H]$.
% \end{definition}
The pseudocode for learner environment interaction is provided below in Protocol \ref{prot:learner_env_interaction}.

\begin{protocol}[!ht]
    \caption{Learner-Environment Interaction}
    \label{prot:learner_env_interaction}
	\begin{algorithmic}
            \STATE parameters: $(\S, \A, H, \P, \phi, s_1; K)$ 
            % \STATE assume: $\P, \phi$ form linear dynamics as per \cref{def:linmdp}.
	    \FOR{$k=1, \ldots, K$}
                \STATE agent decides on a policy $\pi^k$
                % \STATE adversary observes $\pi^k$

                \STATE adversary chooses $H$ cost vectors $\cb{\vc_h^k} \in \R^d$
                
                \STATE define $\l_h^k \colon \S \times \A \to \R$ by $\l_h^k(s, a) = \phi(s, a)\T \vc_h^k$.

	    	\STATE environment resets to $s_1^k = s_1$
		    \FOR{$h=1, \ldots, H$}
			    	\STATE agent observes $s_h^k \in \S$
			    	\STATE agent chooses $a_h^k \sim \pi^k_h(\cdot|s_h^k)$
			    	\STATE agent observes and incurs loss $\l_h^k = \l_h^k(s_h^k, a_h^k)$ 
			    \STATE if $h < H$:
			    	\STATE $\;\;$ environment transitions to $s_{h+1}^k \sim \P_h(\cdot|s, a)$
	        \ENDFOR
        \ENDFOR
	\end{algorithmic}
\end{protocol}

We make the following additional notes with regards to the model we consider: 
(1) for any $s,a \in \S \times \A$, the agent may evaluate $\phi(s, a)$ in $O(1)$ time;
(2) we assume an oblivious and deterministic adversary (but in fact our results hold more generally for the case that the adversary is random, and observes the agent's policies, \emph{but not} trajectory realizations);
(3) with slight overloading of notation, we let $\l_h^k = \l_h^k(s_h^k, a_h^k)$ denote the random loss incurred by the agent on episode $k$ time step $h$. 

\paragraph{Learning objective.}
The expected loss of a policy $\pi$ when starting from state $s\in \S$ at time step $h\in [H]$ is given by the value function;
\begin{align}
	V_h^{\pi}(s; \l) 
	\eqq 
	\E \sbr{\sum_{t=h}^H \l_t(s_t, a_t) \mid s_h = s, \pi, \l},
	\label{eq:V_def}
\end{align}
where we use the extra $(;\l)$ notation to emphasize the specific loss function considered. 
The expected loss conditioned on the agent taking action $a\in \A$ on time step $h$ at $s$ and then continuing with $\pi$ is given by the action-value function;
\begin{align}
	Q_h^{\pi}(s, a; \l) 
	\eqq 
	\E \!\sbr{\sum_{t=h}^H \l_t(s_t, a_t) \mid s_h = s, a_h = a, \pi, \l}\!
	.
	\label{eq:Q_def}
\end{align}

The value and action-value functions of a policy $\pi$ in the MDP $\br{\S, \A, H, \P, \l^k, s_1}$ associated with episode $k\in [K]$ are denoted by, respectively;
\begin{align*}
	V_h^{k, \pi}(s) 
	\eqq 
	V_h^{\pi}(s; \l^k)
	;\; 
	Q_h^{k, \pi}(s, a) 
	\eqq 
	Q_h^{\pi}(s, a; \l^k), 
\end{align*}
where $V_h^{\pi}(s; \l^k)$ and $Q_h^{\pi}(s, a; \l^k)$ have been defined in \cref{eq:V_def,eq:Q_def}.
We let $\pi^\star$ denote the best policy in hindsight;
\begin{align*}
	\pi^\star 
	% \eqq
	% \pi^\star \mid [\pi^1, \ldots, \pi^K]
	\eqq \argmin_{\pi} \cb{
	\sum_{k=1}^K 
		V_1^{k, \pi}(s_1)
	},
\end{align*}
and seek to minimize the \emph{pseudo regret} of the agent policy sequence $\pi^1, \ldots, \pi^K$;
\begin{align}
	\Reg \eqq \sum_{k=1}^K 
	V_1^{k, \pi^k}(s_1) - V_1^{k, \pi^\star}(s_1).
\end{align}
Finally, we note that $\pi^\star$ may depend on player decisions, as the adversary is adaptive.

\paragraph{Additional notation and definitions. }

We let $\norm{\cdot} = \norm{\cdot}_2$ denote the standard Euclidean norm, and
for a positive definite matrix $\Lambda\in \R^{d\times d}$, we let $\norm{v}_\Lambda = \sqrt {v\T \Lambda v}$ denote the weighted norm induced by $\Lambda$.
Further, we let $\norm{\Lambda} = \norm{\Lambda}_{\rm op} = \max_{v, \norm{v}=1} v\T \Lambda v$ denote the operator norm of $\Lambda$.
Finally, we use $\rm{clip}\sbr{x}^a_b \eqq \max\cb{\min\cb{x, a}, b}$ to denote clipping of a real scalar $x$ between $a\in \R$ and $b\in \R$.

\section{Algorithm and Main Result}

The pseudocode for our main algorithm; \textbf{P}olicy \textbf{O}ptimization with \textbf{L}east \textbf{S}quares \textbf{B}onus \textbf{E}xploration, is provided in \cref{alg:polsbe}.
The high level algorithmic template is relatively simple; (1) Rollout $\pi^k$ in the environment; (2) Obtain a (nearly) unbiased estimate $\widehat Q^k$ of $Q^k$; (3) Construct a bonus-to-go estimate $\widetilde B^k$ through least squares policy evaluation in an auxiliary bonus MDP; (4) Perform a mirror-descent update step using the optimistic $Q^k$ function estimate given by $\widehat Q^k - \widetilde B^k$.

 The bonus-to-go estimate is obtained by the least squares policy evaluation subroutine \cref{alg:olspe} (discussed in \cref{sec:olspe}), which outputs an approximation that is optimistic and with bias that can be controlled efficiently. This provides for the major contributing factor in the final regret guarantee; specifically, this approach along with a refined instantaneous $Q$-bonus design allows us to avoid the policy cover used in \citet{luo2021policy}, and leads to a simpler algorithm that explores more efficiently.
 The final bonus function $\widetilde B^k$ encompasses two bonus types; one to compensate for uncertainty in the $Q^k$ estimates
 ($b^k$ in \cref{eq:loss_bonus}), and the other ($b^{\P, k}$ in \cref{alg:olspe}) to compensate for uncertainty in the estimation of the dynamics in the policy evaluation procedure.
 Intuitively speaking, given the agent is at state $s_h$, her bonus for taking action $a_h$ will be high when the expected rollout following $a_h$ traverses state action pairs $(s_t, a_t)$ for which (1) we have poor next state information $s_{t+1}$, and (2) their feature vector $\phi(s_t, a_t)$ points in a direction in the state-action space for which we have poor knowledge of past $Q$-cost vectors $\vq_t^1, \ldots, \vq_t^k$ (these are the low dimensional representations of the $Q$ functions; see \cref{lem:q_vec}).
On a conceptual level, $b^{\P, k}$ drives exploration for the purpose of learning the dynamics, and $b^k$ for the sake of cost function information. 
  
\paragraph{Two-way partitioned blocking.} In order to estimate feature occupancy covariance matrices and Bellman backup operators, \cref{alg:polsbe} plays each policy multiple times.
For a given parameter $\tau \geq 1$, we divide episodes $k \in [K]$ into $\lceil K/(2\tau) \rceil$ blocks, and assume for simplicity of exposition that $K/(2\tau)$ is an integer. We define for all $j \in [K/(2\tau)]$;
\begin{align}
    T_{j, 1} &\eqq \cb{(j-1)\tau + 1, \ldots, j\tau},
    \\
    T_{j, 2} &\eqq \cb{j\tau + 1, \ldots, (j+1)\tau},
    \\
    T_{j\phantom{, 3}} &\eqq T_{j, 1} \cup T_{j, 2}
    \label{eq:block_T_j}.
\end{align}
For all episodes $k\in T_j$ (which we call block $j$), the policy is held fixed and denoted $\pi^{(j)}$. We let $\pi^k$ denote the policy played on episode $k$ throughout, thus $\pi^k = \pi^{(j)}$ for all $k \in T_j$. 
The partitioning of each block into two is done to ensure unbiasedness of the regularized inverse covariance estimator $\smash{ \widehat \Sigma_{kh\gamma}^+ }$, as will be made clear in the analysis. Throughout, we let $\D^k = \cb[s]{\D_h^k}_{h=1}^H$ denote the dataset used for estimations of episode $k$, and slightly abuse notation by referring to it as either containing episode indices, or transition tuples $(s_h^i, a_h^i, s_{h+1}^i)$.

\begin{algorithm}[!ht]
    \caption{PO-LSBE}
    \label{alg:polsbe} 
	\begin{algorithmic}
	    \STATE \textbf{input:} $( \eta, \gamma, \beta, \beta^\P, \epsilon, \sigma^2)$
	    \STATE Set $M=\frac{48 d}{\gamma \sigma}\log\frac{72 d}{\gamma^2 \sigma}$,
	$N = \frac{2}{\gamma}\log\frac{1}{\gamma \epsilon}$, 
	$\tau = M N$. 
	    \STATE Initialize $\pi^{(1)}$ to take actions uniformly at random.
	    \FOR{$j=1, \ldots, \lceil K/(2\tau) \rceil$}
	    	\STATE Play $\pi^k = \pi^{(j)}$ for the $2\tau$ episodes $k\in T_j$ (defined in \cref{eq:block_T_j}),
            and collect $(s_h^k,a_h^k, \l_h^k)_{h\in[H],k\in T_j}$
            \STATE 
            \FOR{$k\in T_j$}
	    	\STATE if $k\in T_{j, 1}$ populate $\D^k$ with $T_{j, 2}$ rollouts
              \STATE otherwise $ (k\in T_{j, 2})$ populate $\D^k$ with $T_{j, 1}$ rollouts
	    	\STATE
                % \STATE For every $h\in [H]:$
	    	\STATE $\widehat \Sigma_{kh\gamma}^+ 
			    	\gets {\rm MGR}(\D_h^k; N, M, \gamma)$ (see \cref{alg:mgr})
			\STATE $\widehat \vq_h^k \gets \widehat \Sigma^+_{kh\gamma} \phi(s_h^k, a_h^k) \sum_{t=h}^H \l_t^k$
			\STATE $\widehat Q_h^k(s, a) = \phi(s, a)\T \widehat \vq_h^k$
%                \STATE 
			\STATE Define the $Q$-bonus by
			\begin{align}
				b_h^k(s, a) =
				\label{eq:loss_bonus}
				\beta \br[B]{ 
						\norm{\phi(s, a)}_{\widehat \Sigma_{kh\gamma}^+}
						+  
						\ab[b]{\pi_h^k(\cdot|s), \norm{\phi(s, \cdot)}_{\widehat \Sigma_{kh\gamma}^+}}
					}
			\end{align}
                \STATE Compute the bonus-to-go with \cref{alg:olspe};
			\begin{align*}
			    \widetilde B^k \gets {\rm OLSPE}(\D^k, b^k; \beta^\P, \beta, \gamma)
			\end{align*}
	    	\ENDFOR
        	
                \STATE Policy improvement step:
        	\begin{align*}
				\pi_h^{(j+1)}(a|s)
				\propto \exp\br{{-\eta \sum_{i=1}^{j} \L_h^{(i)}(s, a)}}
                ,
                \,\text{where }
                \L_h^{(j)}(s, a) 
                &= \frac{1}{\tau} \sum_{k\in T_j} \widehat Q_h^k(s, a) - \widetilde B_h^k(s, a)
			\end{align*}
        \ENDFOR
	\end{algorithmic}
\end{algorithm}
Our main result stated below establishes the regret bound for \cref{alg:polsbe}.
\begin{theorem}
\label{thm:polsbe}
	With an appropriate choice of parameters and  assuming $K=\Omega((d \log d)^2)$, \cref{alg:polsbe} obtains an expected regret guarantee of
	\begin{align*}
		\E\sbr{\Reg} = \widetilde O\br{    
            d H^2 K^{6/7}
            + d^{3/2} H^4 K^{5/7}},
	\end{align*}
	where big-$\widetilde O$ hides constant and logarithmic factors.
\end{theorem}

\subsection{Least Squares Policy Evaluation in Bonus MDPs}
\label{sec:olspe}
The \textbf{O}ptimistic-\textbf{L}east-\textbf{S}quares-\textbf{P}olicy-\textbf{E}valuation (OLSPE) procedure given in
\cref{alg:olspe} is a variant of LSVI-UCB (\citealp{jin2020provably}, see also \citealp{agarwal2019reinforcement}) that is aimed at policy evaluation, and tasked with the computation of the bonus-to-go estimates $\widetilde B^k$. Unlike prior works, we evaluate the policy's bonus (i.e., exploration) coverage, rather than its loss performance (which is estimated separately, in \cref{alg:polsbe}) in an auxiliary \emph{full information} bonus MDP.
Given the immediate $Q$-bonus $b^k$ of episode $k$, we consider the bonus MDP $ \br{\S, \A, H, \P, b^k, s_1}$, which should be interpreted as a \emph{reward} MDP, as the agent will be trying to collect \emph{higher} bonus values.
	It is immediate to see that this \emph{is not} a linear MDP, as the reward function $b^k$ is non-linear. 
	Nonetheless, the dynamics do admit a linear factorization (as per \cref{assume:linmdp}), which allows the use of least squares regression to approximate the value and action-value functions in this MDP.

\begin{algorithm}[!ht]
    \caption{OLSPE($\D^k, b^k; \beta^\P, \beta, \gamma$)} 
    \label{alg:olspe}
	\begin{algorithmic}[1]
            \STATE Set $\lambda = 1$
            
	    \STATE $\widetilde W_{H+1}^k(\cdot) = 0$
	    
	    \FOR {$h=H, \ldots, 1$}
            \STATE $\Lambda_h^{k} 
			\gets \lambda I + \sum\nolimits_{i \in \D_h^k} 
            \phi(s_h^i, a_h^i)\phi(s_h^i, a_h^i)\T$

            \STATE $\widehat \vw_h^k 
			\gets 
			\br{\Lambda_h^k}^{-1} \sum\nolimits_{i \in \D^k_h} \phi(s_h^i, a_h^i)\widetilde W_{h+1}^k(s_{h+1}^i)$

            \STATE $\widetilde \P_h^k \widetilde W_{h+1}^k (s, a) 
    		=  
    		\phi(s, a)\T \widehat \vw_h^k + b_{h}^{\P, k}(s, a)$

            \STATE $B_h^{\max} = 2\beta (H-h+1)/ \sqrt \gamma$

            \STATE $\widetilde B_h^k(s, a) 
			= {\rm clip}\sbr{b_h^k(s, a) 
		+ \widetilde \P_h^k \widetilde W_{h+1}^k(s, a)}_{0}^{B_h^{\max}}$

            \STATE $\widetilde W_{h}^k (s) 
			= \ab{\pi^k(\cdot|s), \widetilde B_h^k(s, \cdot)}$
   
        \ENDFOR 
        \STATE \textbf{return} $\widetilde B^k = \cb[b]{\widetilde B_h^k}_{h\in [H]}$
	\end{algorithmic}
\end{algorithm}
For any policy $\pi$, we denote the true value and action-value functions in the bonus MDP of episode $k$, respectively, by
\begin{align}
    B_h^{k, \pi}(s, a)  &\eqq Q_h^\pi(s, a; b^k),
    \\
    W_h^{k, \pi}(s) &\eqq V_h^\pi(s; b^k)
    \label{eq:W_true_def}.
\end{align}
\cref{alg:olspe}
computes optimistic versions of the above functions for the policy passed as input, which on episode $k$ is always the agent's policy $\pi^k$. 
These are denoted by $\widetilde B^k$ and $\widetilde W^k$, and defined in lines 8 and 9 in the algorithm. % Lines  in \cref{eq:B_tilde_def,eq:W_tilde_def}.
In accordance, we let $\widetilde \P_h^k$ defined in line 6 denote the optimistic estimate of the conditional expectation operator given by the dataset $\D_h^k$. Our notation here is motivated by the true conditional expectation operator $\P_h$;
recall we adopt the convention that $\P_h W (s, a) = \E_{s' \sim \P_h(\cdot|s, a)}W(s')$ for any function $W\colon \S \to \R$. We refer the reader to \cref{sec:appdx_prelim} for further comments regarding this notaiton.

\subsection[Obtaining unbiased Q estimates]{Obtaining unbiased $Q$ estimates}
In order to construct estimates of the loss vector associated with the action-value function of episode $k$ time step $h$, $Q_h^{k, \pi^k}$, we follow prior works and use a linear bandit type estimation procedure \citep[e.g., ][]{dani2007price}. 
Unlike the linear bandit setting, here we do not know the feature occupancy covariance matrix, and moreover it may not be well conditioned.
We address both of these issues in the same natural manner as did \citet{luo2021policy}; we estimate a $\gamma$-regularized version of the inverse covariance using the Matrix Geometric Resampling (MGR) procedure of \citet{neu2020efficient} (see also \citealp{neu2021online}).
Like \citet{luo2021policy}, we employ a version of MGR given in \cref{alg:mgr} that averages over multiple estimators to get better control of the variance of the final output, however we obtain tighter bounds owed to a refined analysis (see \cref{lem:mgr}).
\begin{algorithm}[!ht]
    \caption{MGR ($\D, N, M, \gamma$)} 
    \label{alg:mgr}
	\begin{algorithmic}
		\STATE Set $c=1/2$	
		\STATE Enumerate samples in $\D$ by $\cb{\phi_{m, n}}_{m\in [M],n\in [N]}$
	\STATE Let $A_{m, n} = \gamma I + \phi_{m, n} \phi_{m, n}\T \quad \forall m,n$	
        \FOR {$m=1,\ldots, M$}
			\FOR {$n=1,\ldots, N$}
                    
				\STATE $\widehat \Sigma^{(n)}_{m,\gamma} 
					\gets \prod_{i=1}^n(I - c A_{m, i} )$ 
			\ENDFOR
			\STATE $\widehat \Sigma^+_{m,\gamma} 
					\gets cI + c\sum_{n=1}^N \widehat \Sigma^{(n)}_{m,\gamma} $
		\ENDFOR
		
		\textbf{return} $\widehat \Sigma^+_\gamma =  \frac1M \sum_{m=1}^M \widehat \Sigma^+_{m,\gamma} $
	\end{algorithmic}
\end{algorithm}

\section{The Simulator Setting}
The pseudocode for the simulator version of our method is given in \cref{alg:polsbe_simulator} below. It has the same structure as the simulator based algorithm proposed by \citet{luo2021policy} for the linear-$Q$ setting, only that our bonus-to-go is computed using optimistic approximations via \cref{alg:olspe}.
Notably, the simulator required by our algorithm is weaker than that of \citet{luo2021policy}; we only need to execute agent policies from the initial state $s_1$, but do not require next state samples from arbitrarily chosen state action pairs. Formally, we make the following assumption in this section.

\begin{assumption}
\label{assume:simulator}
    The learner has access to a simulator, which takes a policy $\pi$ as input and returns a trajectory $(s_h, a_h)_{h=1}^H$ sampled from the MDP using $\pi$; $a_h\sim \pi(\cdot|s_h)$, and $s_{h+1} \sim \P_h(\cdot|s_h, a_h)$.
\end{assumption}

We note that \cref{alg:polsbe_simulator} follows the exact same algorithmic design as \cref{alg:polsbe}; only that instead of blocking, the version presented here executes simulator rollouts.
The significance of the result presented next is two-fold. First, it establishes the state-of-the-art regret bound for the simulator setting with a computationally efficient algorithm. 
Second, it demonstrates the guarantee our approach would yield without the limiting factor of the number of online samples; specifically, that given $\widetilde O(K^{4/3})$ additional samples per episode, we arrive at a $\widetilde O(K^{2/3})$ regret bound.

\begin{theorem}
\label{thm:polsbe_sim}
    With an appropriate choice of parameters and  assuming $K=\Omega((d \log d)^2)$, under \cref{assume:simulator},
\cref{alg:polsbe_simulator} obtains an expected regret guarantee of
	\begin{align*}
		\E\sbr{\Reg} = \widetilde O\br{
            H^2 (d K)^{2/3}
            + H^4 (dK)^{1/3}
		}
		,
	\end{align*}
	where big-$\widetilde O$ hides constant and logarithmic factors.
	Furthermore, the number of simulator rollouts required per episode is $\widetilde O(K^{4/3})$.
\end{theorem}

\begin{algorithm}[!ht]
    \caption{PO-LSBE (simulator version)}
    \label{alg:polsbe_simulator}
	\begin{algorithmic}
	    \STATE \textbf{input:} $( \eta, \gamma, \beta, \beta^\P, \epsilon, \sigma^2)$, and a simulator
	    \STATE Set $M=\frac{48 d}{\gamma \sigma}\log\frac{72 d}{\gamma^2 \sigma}$,
	$N = \frac{2}{\gamma}\log\frac{1}{\gamma \epsilon}$, 
	$\tau = d^2 M N$. 
	    \STATE Initialize $\pi^{1}$ to take actions uniformly at random.
	    \FOR{$k=1, \ldots, K$}
	    	\STATE Rollout $\pi^k$ in and collect 
	    		$\cb{(s_h^k,a_h^k,\l_h^k)}_{h=1}^H$

	    	\STATE Populate $\D^k$ with $\tau$ \emph{simulator rollouts} of $\pi^k$
	    	\STATE
	    	\STATE $\widehat \Sigma_{kh\gamma}^+ 
			    	\gets {\rm MGR}(\D_h^k; N, M, \gamma)$ (see \cref{alg:mgr})
			\STATE $\widehat \vq_h^k \gets \widehat \Sigma^+_{kh\gamma} \phi(s_h^k, a_h^k) \sum_{t=h}^H \l_t^k$
			\STATE $\widehat Q_h^k(s, a) = \phi(s, a)\T \widehat \vq_h^k$
                \STATE 
			\STATE Define the $Q$-bonus as in \cref{eq:loss_bonus}
			
                \STATE Compute the bonus-to-go  with \cref{alg:olspe};
			\begin{align*}
			    \widetilde B^k \gets {\rm OLSPE}(\D^k, b^k; \beta^\P, \beta, \gamma)
			\end{align*} 
	    	
        	\STATE Policy improvement step:
        	\begin{align*}
				\pi_h^{k+1}(a|s)
				\propto \exp \br{-\eta \sum_{i=1}^{k} \widehat Q_h^i(s, a) - \widetilde B_h^i(s, a)}
			\end{align*}
        \ENDFOR
	\end{algorithmic}
\end{algorithm}

\section{Analysis Overview}
The analysis makes use of some additional notation described next. 
The state-action occupancy measure induced by a policy $\pi$ on time step $h$ is denoted $d_h^\pi(s, a) = \Pr(s_h=s, a_h=a \mid \pi)$, and with slight overloading $d_h^\pi(s) = \sum_a d_h^\pi(s, a)$ denotes the state occupancy measure.
In sake of conciseness, we let
\begin{align}
    d_h^k \eqq d_h^{\pi^k},
    \;\;
    d_h^\star \eqq d_h^{\pi^\star},
\end{align}
denote the occupancy measures of, respectively, the agent's policy on episode $k$ and the benchmark policy $\pi^\star$.
We let $\E_k\sbr{\cdot} = \E\sbr{\cdot | \pi^k, \ldots, \pi^1}$ denote the expected value of random variables conditioned on the sequence of agent policies up to and including episode $k$; and note this only indicates conditioning on policies and not trajectory rollouts. Finally, we may also use the more compact notation
\begin{align}
    Q_h^k \eqq Q_h^{k, \pi^k},
\end{align}
to refer to the true action-value function of the agent's policy $\pi^k$ in the MDP of episode $k$.

In what follows, we present the high level components of the analysis and provide a proof sketch for \cref{thm:polsbe}; for the full technical details, see \cref{sec:thm_proofs}.
Our high level proof structure is an extended (and slightly reframed) version of the one proposed by \citet{luo2021policy}.
We consider the following regret decomposition;
{%\allowdisplaybreaks%
\begin{align*}
		\Reg = & \underbrace{\sum_{k=1}^K \sum_{h=1}^H \E_{s \sim d_h^\star}\sbr{
			\ab{Q_h^k(s, \cdot) - \widehat Q_h^k(s, \cdot), 
				\pi^k_h(\cdot|s)}
		}}_{\tBone}
		\\
		   + &\underbrace{\sum_{k=1}^K \sum_{h=1}^H \E_{s \sim d_h^\star}\sbr{
			\ab{\widehat  Q_h^k(s, \cdot) - Q_h^k(s, \cdot), 
				\pi^\star_h(\cdot|s)}
		}}_{\tBtwo}
		\\
		 + & \underbrace{\sum_{k=1}^K \sum_{h=1}^H \E_{s \sim d_h^\star}\sbr{
			\ab{\widehat  Q^k_h(s, \cdot) - \widetilde B^k_h(s, \cdot), 
				\pi^k_h(\cdot|s) - \pi^\star_h(\cdot|s)}
		}}_{\tOMD}
		\\
		 + & \underbrace{\sum_{k=1}^K \sum_{h=1}^H \E_{s \sim d_h^\star}\sbr{
			\ab{\widetilde B_h^k(s, \cdot), 
				\pi^k_h(\cdot|s) - \pi^\star_h(\cdot|s)}
		}}_{\tExploration}.
\end{align*}%
}%
An important observation made in \citet{luo2021policy} was that with an appropriate bonus design, the bias and OMD terms contribute $\sum_{k} V^{\pi^\star}(s_1; b^k)$, while the exploration term contributes the \emph{exact negative} of this quantity.
Fortunately, what we will pay for exploration (with a positive term), are the bonuses collected along trajectories of the agent's policy, which may be bounded efficiently.

\paragraph{Bounding the exploration term.}
We begin by establishing confidence bounds on the bonus-to-go estimations computed by \cref{alg:olspe} and defined in \cref{eq:B_tilde_def,eq:W_tilde_def}.
\begin{lemma*}[simplified statement of~\cref{lem:backup_confidence}]
 	For any $\delta > 0$,
 	an appropriate choice of parameters
	ensures that w.p. $\geq 1-\delta$ the following holds for all $k,h,s,a$;
	\begin{align}
		\widetilde B_h^k(s, a) 
		&\geq
		b_h^k(s, a) + \P_h\widetilde W_{h+1}^k(s, a)
            \label{eq:analysis_cb_1}
		\\
		\widetilde B_h^k(s, a) &\leq b_h^k(s, a) + \P_h\widetilde W_{h+1}^k(s, a)
		+ 2 b_h^{\P, k}(s, a)
            \label{eq:analysis_cb_2}
	\end{align}
\end{lemma*}
The proof follows from uniform concentration of the least squares estimates over the class of bonus value functions explored by the algorithm; the arguments are similar in spirit to those made in the work of \citet{jin2020provably}.
Next, we use the confidence bounds to deduce a bound on the exploration term. The lemma below contains a part that is implicit in \citet{luo2021policy} Lemma B.1, and an extension to incorporate the effect of the bonus-to-go approximations.
We note our proof below provides a simpler argument than the original of \citet{luo2021policy}, by offloading most of the technicalities to the extended value difference \cref{lem:extended_value_diff}.

% \begin{lemma*}[Extended value difference, \cite{shani2020optimistic}]
% % \label{lem:extended_value_diff}
% 	Let $M = (\S, \A, H, \P, b)$ be any MDP and $\pi, \pi'$ be any two policies.
% 	Then, for any sequence of functions $\widetilde Q_h^\pi(\cdot, \cdot), \widetilde V_h^\pi (\cdot)$, where $\widetilde V_h^\pi(s) \eqq \ab{\pi_h(\cdot | s), \widetilde Q_h^\pi(s, \cdot)}$ for $h = 1, \ldots, H$, we have
% \begin{align*}
% 	& \widetilde V_1^{\pi} - V_1^{\pi'}
% 	= 
% 	\sum_{h=1}^H \E_{s \sim d_h^{\pi'}}\sbr{
% 		\ab{\widetilde Q^\pi_h(s, \cdot), 
% 		\pi_h(\cdot|s) - \pi'_h(\cdot|s)
% 		}
% 	}
% 	\\
% 	& +
% 	\sum_{h=1}^H \E_{s, a \sim d_h^{\pi'}}\sbr{
% 		\widetilde Q^\pi_h(s, a) - \l_h(s, a)
% 			- \P \widetilde V_{h+1}^{\pi}(s, a)
% 	}.
% \end{align*}
% \end{lemma*}

\begin{lemma*}[compact restatement of~\cref{lem:exploration_1}]
	Assume that both \cref{eq:analysis_cb_1,eq:analysis_cb_2} hold. Then,
	\begin{align}
		\tExploration\leq
    2\sum_{k=1}^K \sum_{h=1}^H \E_{s, a \sim d_h^k}\sbr{
				b_h^{\P, k}(s, a)
				+ b_h^k(s, a)
			} 
		 - \sum_{k=1}^K V_1^{k, \pi^\star}(s_1; b^k)
		.
	\end{align}
\end{lemma*}
\begin{proof}[sketch]
	By the lower bound on $\widetilde B_h^k(s, a)$ \cref{eq:analysis_cb_1}, we have 
	\begin{align*}
    \tExploration 
    &\leq 
		\sum_{k=1}^K \sum_{h=1}^H \E_{s \sim d_h^\star}\sbr{
			\ab{\widetilde B_h^k(s, \cdot), 
				\pi^k_h(\cdot|s) - \pi^\star_h(\cdot|s)}
		}
		\\
		&\quad +
		\sum_{k=1}^K \sum_{h=1}^H \E_{s, a\sim d_h^\star}\sbr{
			\widetilde B_h^k(s, a) - b_h^k(s, a) - \P_h\widetilde W_{h+1}^k(s, a)
		}
		\\
		&= \sum_{k=1}^K \widetilde W_1^k - W_1^{k, \pi^\star}
		,
	\end{align*}
 where the inequality is since we only add non-negative terms, and the equality follows from the extended value difference \cref{lem:extended_value_diff} with $\widehat V^\pi_1 = \widetilde W_1^k = \widetilde W_1^{k, \pi^k}$ and $V_1^{\pi'} = W_1^{k, \pi^\star}$ (and we recall definitions in \cref{eq:W_tilde_def,eq:W_true_def}).
	Next, using \cref{lem:extended_value_diff} again and our upper bound on $\widetilde B_h^k(s, a)$ given by \cref{eq:analysis_cb_2}, establishes that
 $\sum_{k=1}^K \widetilde W^k_1 - W^{k, \pi^k}_1 \leq 2\sum_{k=1}^K \sum_{h=1}^H \E_{s, a \sim d_h^k}\sbr[b]{
				b_h^{\P, k}(s, a)
			}$. Therefore,
	% \begin{align*}
	% 	&\sum_{k=1}^K \widetilde W^k_1 - W^{k, \pi^k}_1 =
	% 	\\
	% 	&\sum_{k=1}^K \sum_{h=1}^H \E_{s, a \sim d_h^k}\sbr{
	% 			 \widetilde B_h^k(s, a) 
 %     - b_h^k(s, a) - \P_h\widetilde W_{h+1}^k(s, a)
	% 		} 
	% 	\\
	% 	&\leq
	% 		2\sum_{k=1}^K \sum_{h=1}^H \E_{s, a \sim d_h^k}\sbr{
	% 			b_h^{\P, k}(s, a)
	% 		}
	% 	.
	% \end{align*}
 %   Now,
        \begin{align*}
            \sum_{k=1}^K \widetilde W_1^k - W_1^{k, \pi^\star}
            &=\sum_{k=1}^K \widetilde W_1^k - W_1^{k, \pi^k}
            + \sum_{k=1}^K W_1^{k, \pi^k} - W_1^{k, \pi^\star}
            \\
            &\leq 
            2\sum_{k=1}^K \sum_{h=1}^H \E_{s, a \sim d_h^k}\sbr{
				b_h^{\P, k}(s, a)
			}
            	+ \sum_{k=1}^K W_1^{k, \pi^k} - W_1^{k, \pi^\star}
          ,
        \end{align*}
        which completes the proof after substituting for the definition of the true bonus value functions \cref{eq:W_true_def}.
\end{proof}
From this point, it is not hard to obtain an in expectation bound; 
\begin{align}
    \E&\sbr{\tExploration}
    \lesssim
    \E\sbr{\sum_{k=1}^K \sum_{h=1}^H \E_{d_h^k}\sbr{
				b_h^{\P, k}(s, a)
				+ b_h^k(s, a)
			} }
    - \E\sbr{\sum_{k=1}^K V_1^{k, \pi^\star}(s_1; b^k)}
    \label{eq:analysis_exploration}.
\end{align}
Notably, the arguments thus far do not depend on the particular form of the immediate bonuses $b^k$, suggesting we would like to choose the bonus so that as much of $\tBone, \tBtwo$ and $\tOMD$ can be expressed as $V_1^{k, \pi^\star}(s_1; b^k)$.
\paragraph[Bounding Bias1 + Bias2]{Bounding $\tBone + \tBtwo$.}
To bound these terms, we employ relatively standard arguments in similar nature to those of \citet{luo2021policy}. However, we aim for a different immediate bonus function, earning important savings in the policy evaluation procedure.
Henceforth, we let
\begin{align}
\Sigma_{kh} \eqq \E_{s, a\sim d_h^k}\sbr{\phi(s, a)\phi(s, a)\T},
\end{align} 
denote the true covariance matrix of the feature occupancy induced by $\pi^k$ on time step $h$, and
denote by $\Sigma_{kh\gamma} \eqq \gamma I + \Sigma_{kh}$ the $\gamma$-regularized version of it.
\begin{lemma*}[simplified restatement of~\cref{lem:bias_bound}]
    For the immediate bonus function $b^k$ defined in \cref{eq:loss_bonus} and an appropriate choice of parameters, 
    we have that the expected bias terms are bounded as
    \begin{align*}
		\E &\sbr{ \tBone + \tBtwo}
		\leq 
	\\
    &\br{\sqrt{\gamma d H^2} } 
        \E\sbr{\sum_{k=1}^K \sum_{h=1}^H
			\E_{s \sim d_h^\star}\sbr{
   \sum_a \br{\pi_h^k(a|s) + \pi^\star_h(a|s)}\norm{\phi(s, a)}_{\widehat \Sigma_{kh\gamma}^+}
   }
		}
		+ 4\epsilon H^2 K.
    \end{align*}
\end{lemma*}
\begin{proof}[sketch]
	Since the MDPs on each episode are linear, we have
	$
		Q_h^{k, \pi^k}(s, a) = \phi(s, a)\T \vq_h^k
	$
	for some $\vq_h^k \in \R^d$ of bounded norm.
	In addition,
	\begin{align*}
		\E_k \sbr{ \widehat \vq_h^k }
		= \E_k \sbr{ \widehat \Sigma^+_{kh\gamma} } 
	    	\Sigma_{kh} \vq_h^k,
	\end{align*}
	and with an appropriate choice of parameters, our inverse covariance estimator is only $\epsilon$-biased (see \cref{lem:mgr}), which can be used to show that
	\begin{align*}
		\E_k \sbr{Q_h^k (s, a) - \widehat Q_h^k(s, a)}
		&= \E_k \sbr{\phi(s, a)\T \br{\vq_h^k - \widehat \vq_h^k}}
		\leq \gamma \phi(s, a)\T \Sigma^{-1}_{kh\gamma} \vq_h^k 
		+ \epsilon H.
	\end{align*}
	Using standard algebraic manipulations, we can further bound the first term appearing on the RHS above by $\sqrt {\gamma d} H \E_k \sbr{ 
			\norm{\phi(s, a)}_{\widehat \Sigma^{+}_{kh\gamma}}} + \epsilon H$, which leads to,
   \begin{align*}
        \E_k &\sbr{Q_h^k (s, a) - \widehat Q_h^k(s, a)}
        \leq 
        \sqrt {\gamma d} H \E_k \sbr{ 
            \norm{\phi(s, a)}_{\widehat \Sigma^{+}_{kh\gamma}} 
        }
        + 2 \epsilon H.
    \end{align*}
    The proof is complete by summing the bounds on the appropriate terms in $\tBone$ and $\tBtwo$, and adding them together.
\end{proof}
From this point, it is not hard to show that owed to our choice of bonus function $b^k$, the result of the above lemma becomes;
\begin{align}
		\E\sbr{ \tBone + \tBtwo}
		\lesssim
		\frac12 \E\sbr{\sum_{k=1}^K
V_1^{\pi^\star}(s_1; b^k)}
		+ \epsilon H^2 K.
    \label{eq:analysis_b1b2}
\end{align}

\paragraph{Bounding $\tOMD$ term.}
The variance of our estimators $\widehat \Sigma_{kh\gamma}^+$ comes into play in the second moment bound derived on the basic mirror-descent guarantee. Using a refined analysis, we show in \cref{lem:mgr} that $\tau = O(1/\gamma^2)$ samples are sufficient to ensure, for $\sigma = 1/4$;
\begin{align*}
    \E \sbr{\widehat \Sigma^+_{kh\gamma} \Sigma_{kh\gamma} \widehat \Sigma^+_{kh\gamma}}
    &\preceq
    2\E \sbr{\widehat \Sigma^+_{kh\gamma}}
    + \sigma I,
\end{align*}
Using the above, we prove;
\begin{lemma*}[simplified restatement of \cref{lem:omd_term_bound_blocking}]
Upon executing \cref{alg:polsbe} with an appropriate choice of parameters, we have for any $s, h$;
	\begin{align*}
		\E &\sbr{ \sum_{k=1}^K 
			\ab{\widetilde Q_h^k(s, \cdot), 
				\pi^k_h(\cdot|s) - \pi^\star_h(\cdot|s)}
		}
		\\
            &\lesssim
		  \frac{\eta H^2}{\sqrt \gamma} \E \sbr{
			\sum_{k=1}^K \sum_a \pi_h^k(a|s)\norm{\phi(s, a)}_{\widehat \Sigma_{kh\gamma}^+}
		}
		% \\
  %           &\quad 
            + \frac{\tau}{\eta}
		+ \frac{\eta \beta^2 H^2 K}{\gamma}
		+ \eta (1 + \sigma) H^2 K 
		.
	\end{align*}
\end{lemma*}
Taken together, these, along with our choice of bonus function $b^k$, establish that
\begin{align}
    \E\sbr{\tOMD} \lesssim \frac12 \E\sbr{\sum_{k=1}^K 
        V_1^{\pi^\star}(s_1; b^k)}
        + \frac{H}{\eta\gamma^2} 
        + \eta H^3 K
        .
    \label{eq:analysis_omd}
\end{align}

\paragraph{Concluding the proof.}
Combining \cref{eq:analysis_exploration,eq:analysis_b1b2,eq:analysis_omd}, and focusing on dependence on $K$, we obtain
\begin{align}
    \E\sbr{\Reg}
    &\lesssim
    \E\sbr{\sum_{k=1}^K \sum_{h=1}^H \E_{d_h^k}\sbr{
        b_h^{\P, k}(s, a)
        + b_h^k(s, a)
    } }
    \nonumber
    % \\
    % &\quad 
    + \frac{1}{\eta \gamma^2} 
    + \eta  K
    + \epsilon K
    .
    \nonumber
\end{align}
We bound the bonus terms collected along the agent's trajectories above using standard arguments in \cref{lem:loss_bonus_bound,lem:dynamics_bonus_bound}, arriving at
\begin{align*}
    \E\sbr{\Reg}
    &\lesssim
    \sqrt \gamma K
    + \frac{1}{\eta \gamma^2} 
    + \eta K
    + \epsilon K
    .
\end{align*}
We can easily rid of the bias term $\epsilon K$ as $\tau$ depends on it only logarithmically. Finally, the first two terms dominate the regret at $\widetilde O(K^{6/7})$ for the setting of $\eta=\gamma/(2H)$, and $\gamma=K^{-2/7}$, and the proof is complete.

\subsection*{Acknowledgements}
The authors would like to thank Asaf Cassel for many helpful discussions.
This work was supported by the European Research Council (ERC) under the European Union’s Horizon 2020 research and innovation program (grant agreement No.~882396), by the Israel Science Foundation (grants number 993/17, 2549/19), by the Len Blavatnik and the Blavatnik Family foundation, by the Yandex Initiative in Machine Learning at Tel Aviv University, 
by a grant from the Tel Aviv University Center for AI and Data Science (TAD). 

\bibliography{main}

%%%%%%%%%%%%%%%%%%%%%%%%%%%%%%%%%%%%%%%%%%%%%%%%%%%%%%%%%%%%%%%%%%%%%%%%%%%%%%%
%%%%%%%%%%%%%%%%%%%%%%%%%%%%%%%%%%%%%%%%%%%%%%%%%%%%%%%%%%%%%%%%%%%%%%%%%%%%%%%
% APPENDIX
%%%%%%%%%%%%%%%%%%%%%%%%%%%%%%%%%%%%%%%%%%%%%%%%%%%%%%%%%%%%%%%%%%%%%%%%%%%%%%%
%%%%%%%%%%%%%%%%%%%%%%%%%%%%%%%%%%%%%%%%%%%%%%%%%%%%%%%%%%%%%%%%%%%%%%%%%%%%%%%
\newpage
\appendix

\section{Analysis Preliminaries}
\label{sec:appdx_prelim}
For convenience, the table below summarizes most of the notation used throughout the analysis.
\begin{center}
\begin{tabular}{ c | l}
	$\P_h (\cdot|s, a)$
	& 
	The density function of the next state given the agent is at $s$ and acts $a$
        \\
	$\P_h V \colon \S \times \A \to \R$
	& 
	For any function $V\colon \S \to \R$, defined by $\P_h V (s, a) = \E_{s' \sim \P_h(\cdot|s, a)}V(s')$
        \\
    $\E_k[\cdot]$
        &
        Expectation conditioned on past policies;
        $\E_k[\cdot] \eqq \E\sbr{\cdot | \pi^1, \ldots, \pi^k}$
	\\	
	$d_h^k, d_h^\star$
	& 
	State and state-action occupancy measures of $\pi^k, \pi^\star$.
	\\	
 $D^k = \cb{\D_h^k}$
	& 
	The dataset used to compute $\widetilde B^k$ and $\widehat \Sigma_{kh\gamma}^+$.
    % Contains episode indices / tuples $\br{s_h^i, a_h^i, s_{h+1}^i}$
    \\
	$\vc_h^k \in \R^d$
	& 
	The adversarially chosen cost vector of episode $k$
        \\
	$\l_h^k(s, a)$
	& 
	The loss function of episode $k$ applied to $s, a$; $\l_h^k(s, a) = \phi(s, a)\T \vc_h^k$
        \\
        $\l_h^k \in \R$
	& 
	Loss of the agent on episode $k$ time $h$; $\l_h^k = \l_h^k(s_h^k, a_h^k)$ 
 % (slight notation overloading here)
        \\
	$Q^{k ,\pi}$
	& 
	The $Q$ function of policy $\pi$ in the MDP of episode $k$
	\\
	$Q^{k}$
	& 
	The true $Q$ function of policy $\pi^k$ in the MDP of episode $k$
        \\	
	$\vq_h^k \in \R^d$
	& 
	The low dimensional representation of $Q_h^{k, \pi^k}$
	\\
	$\widehat \vq_h^k \in \R^d$
	& 
	(nearly) unbiased estimate of $\vq_h^k$, see \cref{alg:polsbe}
	\\
	$\widehat Q_h^k$
	& 
	(nearly) unbiased estimate of $Q_h^{k, \pi^k}$; $\widehat Q_h^{k} (s, a)=\phi(s, a)\T \widehat \vq_h^k$; see \cref{alg:polsbe}
	\\
	$b_h^k$
	& 
	Immediate bonus (also referred to as $Q$-bonus) function; see \cref{alg:polsbe}
	\\
	$b_h^{\P, k}$
	& 
	Dynamics bonus function, used for bonus-to-go optimism; see \cref{alg:olspe}
	\\
	$B_h^{k, \pi}$
	& 
	True bonus-to-go function in the bonus MDP 
	$B_h^{k, \pi}(s, a) =  Q_h^\pi(s, a; b^k)$ 
	\\	
	$W_h^{k, \pi}$
	& 
	True value function in the bonus MDP; $W_h^{k, \pi}=V_h^\pi(s; b^k)$
	\\
	$\widetilde B_h^k$
	& 
	The optimistic approximation of $B^{k, \pi^k}_h$ 
 ; see \cref{alg:olspe}
	\\
	$\widetilde W_h^k$
	& 
	The optimistic approximation of $W^{k, \pi^k}_h$; see \cref{alg:olspe}
	\\
	$\widetilde P_h^k \widetilde W_{h+1}^k\colon \S \times \A \to \R$
	& 
	The optimistic approximation of $\P_h \widetilde W^{k}_{h+1}$; see \cref{alg:olspe}
	\\
	$\Lambda_h^k \in \R^{d\times d}$
	& 
	Empirical non-normalized covariance of $d_h^k$;
        see \cref{alg:olspe}
	% $\Lambda_h^k = \lambda I 
	% + \sum_{i \in \D_h^k} \phi(s_h^i, a_h^i)\phi(s_h^i, a_h^i)\T$;
        \\
	$\widehat \vw_h^k \in \R^d$
	& 
	Estimate of the low dimensional representation of $\P_h \widetilde W_{h+1}^k$;
    see \cref{alg:olspe}
	\\
	$\Sigma_{kh} \in \R^{d \times d}$
	& 
	Feature occupancy covariance; 
	$\Sigma_{kh} = \E_{s, a\sim d_h^k}\sbr{\phi(s, a) \phi(s, a)\T }$
	\\
	$\Sigma_{kh\gamma} \in \R^{d \times d}$
	& 
	$\gamma$-regularized feature occupancy covariance; 
	$\Sigma_{kh\gamma} = \gamma I + \Sigma_{kh}$
	\\
	$\widehat \Sigma_{kh\gamma}^+ \in \R^{d \times d}$
	& 
	(nearly) unbiased estimate of $\Sigma_{kh\gamma}^{-1}$, computed by \cref{alg:mgr}
        \\
	$\lambda$
	& 
	Regularization parameter for LSVI in \cref{alg:olspe}, fixed to $\lambda=1$ throughout.
        \\
	$\gamma$
	& 
	Regularization parameter for inverse covariance estimation, see \cref{alg:mgr}
        \\
	$\beta$
	& 
	$Q$-bonus function factor (see \cref{eq:loss_bonus})
        \\
	$\beta^\P$
	& 
	Dynamics bonus function factor (see \cref{eq:dynamics_bonus_def})
	\end{tabular}
\end{center}

\paragraph{Notation for conditional expectation operators.}
We use the convention that for any function $V \colon \S \to \R$, the conditional expectation operator is denoted by $\P_h$;
\begin{align}
    \P_h V (s, a) \eqq \E_{s' \sim \P_h(\cdot|s, a)} V(s').
\end{align}
We note the motivation for this notation comes from considering (when the state space is finite) the matrix $\P_h \in \R^{S A \times S}$ where $S = |\S|$, and the vector $V \in \R^{S}$. Then the result of multiplying them is indeed a vector $\P_h V \in \R^{S A}$ with $\P_h V (s, a) = \sum_{s'} \P_h(s'|s, a)V(s') = \E_{s' \sim \P_h(\cdot|s, a)} V(s')$. 
In similar spirit and with slight abuse of notation, we let $\widetilde \P_h^k \colon \R^{S} \to \R^{SA}$ denote an optimistic conditional expectation that \emph{is not} a linear operator, but rather defined by;
\begin{align*}
    \widetilde \P_h^k W (s, a) \eqq 
    (\widetilde \P_h^k W) (s, a) \eqq 
    \widehat \P_h^k W (s, a) + b_h^{\P, k}(s, a),
    \text{ where }
    \widehat \P_h^k 
    \eqq 
    \br{\Lambda_h^k}^{-1} \sum\nolimits_{i \in \D^k_h} \phi(s_h^i, a_h^i) \ve[s_{h+1}^i]\T,
\end{align*}
where $\ve[s]$ denotes the $s$'th standard basis vector in $\R^S$. Thus, the $\widetilde \P_h^k$ operator is composed from a linear one $\widehat \P_h^k$ plus a bonus term. The above decomposition is discussed to motivate our notation, but otherwise is not needed anywhere in our proofs as we always apply $\widetilde \P_h^k$ to $\widetilde W_{h+1}^k$.

\paragraph{Definitions from \cref{alg:olspe}.}
Below, we repeat definitions made in \cref{alg:olspe} that will be referred to throughout the analysis.

\begin{align}
	  %   	&\Lambda_h^{k} 
			% \gets \lambda I + \sum\nolimits_{i \in \D_h^k} 
   %          \phi(s_h^i, a_h^i)\phi(s_h^i, a_h^i)\T
   %          	\nonumber
   %          	\\
   %          	&\widehat \vw_h^k 
			% \gets 
			% \br{\Lambda_h^k}^{-1} \sum\nolimits_{i \in \D^k_h} \phi(s_h^i, a_h^i)\widetilde W_{h+1}^k(s_{h+1}^i)
			% \nonumber
			% \\
			&b_h^{\P, k}(s, a) 
			= 
			\beta^{\P} \norm{\phi(s, a)}_{\br{\Lambda_h^k}^{-1}}
                \label{eq:dynamics_bonus_def}
			\\
			&\widetilde \P_h^k \widetilde W_{h+1}^k (s, a) 
    		=  
    		\phi(s, a)\T \widehat \vw_h^k + b_{h}^{\P, k}(s, a)
      \label{eq:backup_dynamics_def}
    		\\
                &B_h^{\max} = 2\beta (H-h+1)/ \sqrt \gamma
                \nonumber
			\\
   &\widetilde B_h^k(s, a) 
			= {\rm clip}\sbr{b_h^k(s, a) 
		+ \widetilde \P_h^k \widetilde W_{h+1}^k(s, a)}_{0}^{B_h^{\max}}
  \label{eq:B_tilde_def}
			\\
			&\widetilde W_{h}^k (s) 
			= \ab{\pi^k(\cdot|s), \widetilde B_h^k(s, \cdot)}
   \label{eq:W_tilde_def} 
		\end{align}

\paragraph{Bellman consistency equations.}
The value and action-value functions, in any MDP, satisfy;
\begin{align}
    Q_h^\pi &= \l_h + \P_h V_{h+1}^{\pi}
    \\
    V_h^{\pi}(s) 
    &= \ab{\pi(\cdot|s), Q_h^{\pi}(s, \cdot)}
\end{align}

\paragraph{Preliminary lemmas.}

\begin{lemma}
\label{lem:q_vec}
	Let $\M = \br{\S, \A, H, \P, \l}$ be any linear MDP (see \cref{def:linmdp}) with $\l_h(s, a)=\phi(s, a)\T \vc_h$ for cost vectors $\cb{\vc_h}_{h=1}^H \subset \R^d$. Then, for any policy $\pi$ and time step $h$, there exists $\vq_h^\pi \in \R^d$ such that $Q_h^\pi(s, a) = \phi(s, a)\T \vq_h^\pi$.
		Furthermore, $\norm{\vq_h^\pi} \leq H\sqrt d$.
\end{lemma}
\begin{proof}
	Observe;
	\begin{align*}
        Q_h^{\pi}(s, a)
        = \l_h(s, a) 
        + \E \sbr{V_{h+1}^{\pi}(s_{h+1}) \mid s_h = s, a_h=a}
        = \phi(s, a)\T \br{ \vc_h + \int \psi_h(s')
            V_{h+1}^{\pi}(s') {\rm d} s'},
    \end{align*}
    thus the first claim follows with $\vq_h^\pi \eqq \vc_h + \int \psi_h(s')
            V_{h+1}^{\pi}(s') {\rm d} s' $. For the second part, 
            note that 
    \begin{align*}
    	\norm{\vq_h^\pi} 
    	= \norm{\vc_h + \int \psi_h(s')
            V_{h+1}^{\pi}(s') {\rm d} s'}
        \leq \sqrt d + \sqrt d \norm{V_{h+1}^{\pi}}_\infty
        \leq \sqrt d + \sqrt d (H - 1) = H \sqrt d,
    \end{align*}
    where the first inequality follows by assumption (see \cref{def:linmdp}).
\end{proof}
In what follows we will refer to the true low dimensional $Q$-vector on episode $k$ time step $h$;
\begin{align}
	\vq_h^k 
	\eqq 
	\vq_h^{k, \pi^k}
	\eqq  \vc_h^k + \int \psi_h(s')
            V_{h+1}^{k, \pi^k}(s') {\rm d} s'.
    \label{eq:q_vec_def}
\end{align}
By \cref{lem:q_vec}, we have that $\norm{\vq_h^k} \leq H\sqrt d$, and 
\begin{align*}
	Q_h^{k, \pi^k}(s, a) = \phi(s, a)\T \vq_h^k,
\end{align*}
for all $s, a, h, k$.

\begin{lemma}
\label{lem:independence_base}
	In both \cref{alg:polsbe,alg:polsbe_simulator}, it holds that for all $h\in[H], k\in [K]$, conditioned on $\pi^1, \ldots, \pi^k$, we have that $\vq_h^k$ is fixed,
 and that $\widehat \Sigma_{kh\gamma}^+$ and $\br{s_t^k, a_t^k, \l_t^k}_{t=1}^H$ are independent.
\end{lemma}
\begin{proof}
	First note that a-priori $\vq_h^k$ is a random variable determined by the adversary's choice of cost vectors on episode $k$, which may depend on $\pi^1, \ldots, \pi^k$.
	However, when conditioning on $\pi^1, \ldots, \pi^k$ the adversary's (which we assume is deterministic) is clearly fixed. 
 
    For the second part in the claim, consider first \cref{alg:polsbe_simulator}, where $\widehat \Sigma_{kh\gamma}^+$ is computed from samples generated by the simulator. Thus it immediately follows that $\widehat \Sigma_{kh\gamma}^+$ and $\br{s_t^k, a_t^k, \l_t^k}_{t=1}^H$ are indeed independent conditioned on $\pi^k$, for all $h, k$.
	
	For \cref{alg:polsbe}, let $k, h$, such that $k \in T_j$, and note that $\cb{\pi^1, \ldots, \pi^k}$ are in fact just $\cb{\pi^{(1)}, \ldots, \pi^{(j)}}$. 
    Conditioning on $\pi^k$, all rollouts in block $j$ are independent. 
    In addition, transitions of episode $k$ \emph{are not} contained in $\D_h^k$ (by the two-way block partitioning \cref{eq:block_T_j}).
 Thus, conditioning on $\pi^{k} = \pi^{(j)}$, this immediately implies $\widehat \Sigma_{kh\gamma}^+$ (which is computed only from samples in $\D_h^k$) and $\br{s_t^k, a_t^k, \l_t^k}_{t=1}^H$ are indeed independent, and completes the proof.
\end{proof}

\section{Theorem Proofs}

\label{sec:thm_proofs}

The analysis begins by considering a slightly reframed version of the regret decomposition proposed by \citet{luo2021policy};
    \begin{align}
		\text{Regret}
		=\sum_{k=1}^K \sum_{h=1}^H &\E_{s \sim d_h^\star}\sbr{
			\ab{Q_h^k(s, \cdot), 
				\pi^k_h(\cdot|s) - \pi^\star_h(\cdot|s)}
		}
            \nonumber
		\\
		&= \underbrace{\sum_{k=1}^K \sum_{h=1}^H \E_{s \sim d_h^\star}\sbr{
			\ab{Q_h^k(s, \cdot) - \widehat Q_h^k(s, \cdot), 
				\pi^k_h(\cdot|s)}
		}}_{\tBone}
            \nonumber
		\\
		&\quad + \underbrace{\sum_{k=1}^K \sum_{h=1}^H \E_{s \sim d_h^\star}\sbr{
			\ab{\widehat  Q_h^k(s, \cdot) - Q_h^k(s, \cdot), 
				\pi^\star_h(\cdot|s)}
		}}_{\tBtwo}
		\nonumber
            \\
		&\quad + \underbrace{\sum_{k=1}^K \sum_{h=1}^H \E_{s \sim d_h^\star}\sbr{
			\ab{\widehat  Q_h^k(s, \cdot) - \widetilde B_h^k(s, \cdot), 
				\pi^k_h(\cdot|s) - \pi^\star_h(\cdot|s)}
		}}_{\tOMD}
            \nonumber
		\\
		&\quad + \underbrace{\sum_{k=1}^K \sum_{h=1}^H \E_{s \sim d_h^\star}\sbr{
			\ab{\widetilde B_h^k(s, \cdot), 
				\pi^k_h(\cdot|s) - \pi^\star_h(\cdot|s)}
		}}_{\tExploration}
  \label{eq:regret_decomposition}
	\end{align}

Next, we will state the relevant lemmas used to bound each of the terms, and then proceed to the main proof.
All subsequent arguments hinge on peroperties of our inverse covariance estimators, which are stated in the below lemma, and proved in \cref{sec:proofs:mgr}.
\begin{lemma}[MGR]
\label{lem:mgr}
	Let $\epsilon, \sigma, \gamma > 0$ be three parameters and assume also $\sigma \leq 1/4$, $\epsilon \leq \sigma/6$ and that $\gamma < 1/2$.
	Assume $\D$ contains $MN$ i.i.d.~samples $\cb{\phi} \subset \R^d$, $\norm{\phi} \leq 1$, from some distribution $p$, and let $\Sigma_{\gamma} \eqq \E_{\phi \sim p}\sbr{\phi \phi\T} + \gamma I$.
	Then invoking \cref{alg:mgr} with arguments $(\D, M, N, \gamma)$, for $M=\frac{48 d}{\gamma \sigma}\log\frac{72 d}{\gamma^2 \sigma}$ and 
	$N = \frac{2}{\gamma}\log\frac{1}{\gamma \epsilon}$, we have
	\begin{align}
		\norm{\widehat \Sigma^{+}_{\gamma}} 
		&\leq \frac{1}{\gamma}
		\text{ almost surely,} 
		\label{eq:mgr_norm}
		\\
		\norm[B]{\E\sbr{\widehat \Sigma_{\gamma}^+} - \Sigma_{\gamma}^{-1}}
		&\leq \epsilon,
		\label{eq:mgr_bias}
		\\
		\E\sbr{\widehat \Sigma^+_\gamma \Sigma_\gamma \widehat \Sigma^+_\gamma}
		&\preceq
		2\E \sbr{\widehat \Sigma^+_\gamma}
		+ \sigma I.
		\label{eq:mgr_variance}
	\end{align}
\end{lemma}

To bound the exploration term, we intially establish confidence bounds on our approximate bouns-to-go functions.
\begin{lemma}[Bonus backup confidence bounds]
\label{lem:backup_confidence}
 	Assume $\beta = 2 H \sqrt {\gamma d}, \lambda \geq 1, \gamma \geq 1/K$, $|\D_h^k| = \widetilde O((dHK)^4)$, and $\norm[b]{\widehat \Sigma^+_{kh\gamma}}\leq 1/\gamma$ for all $k, h$.
	Then, there exists a universal constant $C_1$, such that
	for any $\delta > 0$, setting 
	$\beta^\P \geq C_1 H^2d^{3/2}\log\br{
			d \beta K H /\delta
		}$ ensures that w.p. $\geq 1-\delta$ the following holds for all $k,h,s,a$;
	\begin{align}
		b_h^k(s, a) + \P_h\widetilde W_{h+1}^k(s, a)
		\leq
		\widetilde B_h^k(s, a) 
		&\leq b_h^k(s, a) + \P_h\widetilde W_{h+1}^k(s, a)
		+ 2 b_h^{\P, k}(s, a)
		,
        \label{eq:backup_confidence_bounds}
	\end{align}
    where $\widetilde B_h^k, \widetilde W_h^k$ are defined in \cref{eq:B_tilde_def,eq:W_tilde_def}.
\end{lemma}
The proof of \cref{lem:backup_confidence} follows from uniform concentration over the class of bonus value functions explored by our algorithm. The arguments are in the spirit of those given in \citet{jin2020provably}, and is deferred to \cref{sec:proofs:backup_confidence}.
With the above confidence bounds in place, the exploration term bound follows from the next lemma (for proof see \cref{sec:proofs:exploration}).
\begin{lemma}
\label{lem:exploration_2}
    Assume the backup confidence bounds \cref{eq:backup_confidence_bounds} hold with probability at least $1-\delta$, where $\delta \leq (7 K H^2 (\beta/\sqrt \gamma + \beta^\P/\sqrt \lambda))^{-1}$.
    Then expected exploration term is bounded as
	\begin{align*}
		&\E\sbr{\sum_{k=1}^K \sum_{h=1}^H \E_{s \sim d_h^\star}\sbr{
			\ab{\widetilde B_h^k(s, a), 
				\pi^k_h(\cdot|s) - \pi^\star_h(\cdot|s)}
		}}
		\\
		&\leq 
		2\E\sbr{\sum_k \sum_{h=1}^H \E_{s, a \sim d_h^k}\sbr{
				b_h^{\P, k}(s, a)
				+ b_h^k(s, a)
			} }
		- \E\sbr{\sum_{k=1}^K \sum_{h=1}^H 
			\E_{s, a \sim d_h^\star}\sbr{b_h^k(s, a)}
			}
		+ 1
		.
	\end{align*}
\end{lemma}

The two final important lemmas we state before turning to the proof are given next; these bound, respectively, the OMD and bias terms. We defer proofs of both to \cref{sec:proofs:bias_omd}.
\begin{lemma}[\cref{alg:polsbe} OMD term bound]
\label{lem:omd_term_bound_blocking}
Assume that \cref{alg:polsbe} is executed
   with $\eta \leq \gamma /(2H)$, $\beta \leq 1/2\sqrt \gamma$, and $\gamma \leq 1$. 
  Further, assume that for all $k, h$;
   $\E\sbr{\widehat \Sigma^+_{kh\gamma} \Sigma_{kh\gamma} \widehat \Sigma^+_{kh\gamma}}
		\preceq
		2\E \sbr{\widehat \Sigma^+_{kh\gamma}}
		+ \sigma I$, and
 $\norm[b]{\widehat \Sigma_{kh\gamma}^+} \leq 1/\gamma$ almost surely.

 Then, we have for any $s, h$;
	\begin{align*}
		\E &\sbr{ \sum_{k=1}^K 
			\ab{\widehat Q_h^k(s, \cdot) - \widetilde B_h^k(s, \cdot), 
				\pi^k_h(\cdot|s) - \pi^\star_h(\cdot|s)}
		}
		\\
		&\leq
		\frac{\tau \log A}{\eta}
		+ \frac{8\eta \beta^2 H^2 K}{\gamma}
		+ \frac{2\eta H^2}{\sqrt \gamma} \E \sbr{
			\sum_{k=1}^K \sum_a \pi_h^k(a|s)\norm{\phi(s, a)}_{\widehat \Sigma_{kh\gamma}^+}
		}
		+ 2 \eta (1 + \sigma) H^2 K 
		+ \frac{2 \tau H}{\gamma}.
	\end{align*}
\end{lemma}

\begin{lemma}[Bias bound]
\label{lem:bias_bound}
	Assuming $\norm[b]{ \E_k\sbr{\widehat \Sigma_{kh\gamma}^+} - \Sigma_{kh\gamma}^{-1}} \leq \epsilon$, $\norm[b]{\widehat \Sigma_{kh\gamma}^+} \leq 1/\gamma$ for all $h, k$, and $\gamma \leq 1/\sqrt d$, we have
	\begin{align*}
		\E&\sbr{ \tBone + \tBtwo}
		\leq 
		\\ &\br{\sqrt{\gamma d H^2} } \E\sbr{\sum_{k=1}^K \sum_{h=1}^H
			\E_{s \sim d_h^\star}\sbr{\sum_a \br{\pi_h^k(a|s) + \pi^\star_h(a|s)}\norm{\phi(s, a)}_{\widehat \Sigma_{kh\gamma}^+}}
		}
		+ 4\epsilon H^2 K.
	\end{align*}
\end{lemma}

\subsection[Thm 1]{\cref{thm:polsbe} proof}

\begin{proof}[of \cref{thm:polsbe}]
We shall use the following set of parameters;
$\sigma=1/4, 
\beta=2H\sqrt{d \gamma}, 
\epsilon=1/K,
\eta=\gamma/(2H),
\gamma = K^{-2/7}$,
 and 
	$
	\beta^\P=10 C_1 H^2d^{3/2}\log\br{
			28 C_1 d \beta K H
	}$, where the constant $C_1$ is that specified by \cref{lem:backup_confidence}.
    
 By our setting of $\tau = M N$
 in the algorithm, each estimation dataset is of size $|\D_h^k| = \frac{48 d}{\gamma \sigma}\log\frac{72 d}{\gamma^2 \sigma} \times \frac{2}{\gamma}\log\frac{1}{\gamma \epsilon}$.
 This, as well as \cref{lem:independence_base} and our parameter choices imply the conditions for \cref{lem:mgr} are met, thus
 it follows that for all $h, k$  \cref{eq:mgr_norm,eq:mgr_bias,eq:mgr_variance} hold for $\widehat \Sigma_{\gamma}^+ = \widehat \Sigma_{kh\gamma}^+, \Sigma_\gamma = \Sigma_{kh\gamma}$.
 Proceeding, we begin by bounding the bias and OMD terms of \cref{eq:regret_decomposition}. From \cref{lem:omd_term_bound_blocking}, we immediately get that
\begin{align*}
    \E\sbr{\tOMD}
    &\leq
    + \frac{2\eta H^2}{\sqrt \gamma} \E \sbr{
        \sum_{k=1}^K \sum_{h=1}^H \E_{s\sim d_h^\star} \sbr{\sum_a \pi_h^k(a|s)\norm{\phi(s, a)}_{\widehat \Sigma_{kh\gamma}^+}}
    }
    \\
    &\quad + \frac{\tau H \log A}{\eta}
    + \frac{8\eta \beta^2 H^3 K}{\gamma}
    + 2 \eta (1 + \sigma) H^3 K 
    + \frac{2 \tau H^2}{\gamma}
    .
\end{align*}
Combining the above with \cref{lem:bias_bound}
	and setting
	\begin{align}
	\mathcal E \eqq \frac{\tau H \log A}{\eta} 
		+ 4 \epsilon H^2 K 
		+ \frac{8\eta \beta^2 H^3 K}{\gamma}
		+ 2 \eta H^3 K(1 + \sigma)
            + \frac{2 \tau H^2}{\gamma}
	,
	\label{eq:main_proof_oK}
	\end{align}
	we have,
	\begin{align*}
		\E &\sbr{\tBone + \tBtwo + \tOMD}
		\\
		&\leq
		\br{\sqrt{\gamma d H^2} } \E\sbr{\sum_{k=1}^K \sum_{h=1}^H
			\E_{s \sim d_h^\star}\sbr{\sum_a \br{\pi_h^k(a|s) + \pi^\star_h(a|s)}\norm{\phi(s, a)}_{\widehat \Sigma_{kh\gamma}^+}}
		}
		\\
		&\quad 
		+ \frac{2\eta H^2}{\sqrt \gamma} \E \sbr{
			\sum_{k=1}^K \sum_{h=1}^H \E_{s \sim d_h^\star}\sbr{
			\sum_a \pi_h^k(a|s)\norm{\phi(s, a)}_{\widehat \Sigma_{kh\gamma}^+}}
		}
		+ \mathcal E
		\\
		% (2)
		&\leq 
		\br{\sqrt{ \gamma d H^2} + \frac{2\eta H^2}{\sqrt \gamma}}
		\E\sbr{\sum_{k=1}^K \sum_{h=1}^H
			\E_{s \sim d_h^\star}\sbr{
			\sum_a \br{\pi_h^k(a|s) + \pi^\star_h(a|s)}\norm{\phi(s, a)}_{\widehat \Sigma_{kh\gamma}^+}}
		}
		+ \mathcal E
		\\
		% (2)
		&\leq 
		\beta
		\E\sbr{\sum_{k=1}^K \sum_{h=1}^H
			\E_{s \sim d_h^\star}\sbr{
			\sum_a \br{\pi_h^k(a|s) + \pi^\star_h(a|s)}\norm{\phi(s, a)}_{\widehat \Sigma_{kh\gamma}^+}}
		}
		+ \mathcal E
		,
	\end{align*}
	where the last inequality follows from our setting of $\eta$ and $\beta$;
	\begin{align*}
		\sqrt{ \gamma d H^2} + \frac{2\eta H^2}{\sqrt \gamma} 
		= 
		\sqrt{ \gamma d H^2} + \sqrt \gamma H 
		\leq 2 H \sqrt{ \gamma d} = \beta.
	\end{align*}
	Further, note that by our $Q$-bonus definition (see \cref{eq:loss_bonus}),
	\begin{align*}
		\beta \sum_a \br{\pi_h^k(a|s) + \pi^\star_h(a|s)}\norm{\phi(s, a)}_{\widehat \Sigma_{kh\gamma}^+}
		&=
		\beta \sum_a 
		\pi^\star_h(a|s)\norm{\phi(s, a)}_{\widehat \Sigma_{kh\gamma}^+}
		+\beta \sum_{a}\pi_h^k(a|s)\norm{\phi(s, a)}_{\widehat \Sigma_{kh\gamma}^+}
		\\
		&=
		\beta \sum_a 
		\pi^\star_h(a|s)\br{\norm{\phi(s, a)}_{\widehat \Sigma_{kh\gamma}^+}
		+\sum_{a'}\pi_h^k(a'|s)\norm{\phi(s, a')}_{\widehat \Sigma_{kh\gamma}^+}}
		\\
		&=
		\sum_a 
		\pi^\star_h(a|s)b_h^k(s, a)
		,
	\end{align*}
	therefore,
	\begin{align*}
		\E\sbr{\tBone + \tBtwo + \tOMD}
		\leq 
		\E\sbr{\sum_{k=1}^K \sum_{h=1}^H
			\E_{(s,a) \sim d_h^\star}\sbr{
			b_h^k(s, a)
			}
		}
		+ \mathcal E
		.
	\end{align*}

 Next, for the exploration term in \cref{eq:regret_decomposition}, 
        first observe our choice of $\beta^\P$ is such that
        $
            \beta^\P 
            \geq C_1 H^2d^{3/2}\log\br{
                    d \beta K H /\delta
            }
        $
        for $\delta = (28 C_1 K H d)^{-9}$.
        In addition, our choice of parameters is such that
        $\delta \leq \br[b]{7 K H^2 (\beta/\sqrt \gamma + \beta^\P/\sqrt \lambda)}^{-1}$, and
        for all $h, k$; $|\D_h^k| = \widetilde O(d K)$.
        Thus, we may invoke \cref{lem:backup_confidence} which ensures the backup confidence bounds \cref{eq:backup_confidence_bounds} hold w.p.~$\geq 1-\delta$, and by \cref{lem:exploration_2},
        this now implies that
	\begin{align}
		\E\sbr{\tExploration}
		\leq 
		2\E\sbr{\sum_k \sum_{h=1}^H \E_{s, a \sim d_h^k}\sbr{
				b_h^{\P, k}(s, a)
				+ b_h^k(s, a)
			} }
		- \E\sbr{\sum_{k=1}^K \sum_{h=1}^H 
			\E_{s, a \sim d_h^\star}\sbr{b_h^k(s, a)}
			}
		+ 1
		.
	\end{align}

 %        \begin{align}
	% 	\E\sbr{\tExploration}
	% 	\lesssim 
	% 	- \E\sbr{\sum_{k=1}^K V_1^{\pi^\star} (s_1; b^k)
	% 		}
 %            + \E\sbr{\sum_k \sum_{h=1}^H \E_{s, a \sim d_h^k}\sbr{
	% 			b_h^{\P, k}(s, a)
	% 			+ b_h^k(s, a)
	% 		} }
	% 	.
	% \end{align}
 
	Combining the the last two displays, we obtain;
	\begin{align*}
		\E\sbr{\Reg}
		&\leq 
		2\E\sbr{\sum_k \sum_{h=1}^H \E_{s, a \sim d_h^k}\sbr{
				b_h^{\P, k}(s, a)
				+ b_h^k(s, a)
			} }
		+ \mathcal E  + 1
		.
	\end{align*}
	To finish the proof, by \cref{lem:dynamics_bonus_bound_base}, and that $|\D_h^k| = \tau$;
	\begin{align*}
		\E\sbr{\sum_k \sum_{h=1}^H \E_{s, a \sim d_h^k}\sbr{
				b_h^{\P, k}(s, a)
			} }
		\leq 
		\frac{20 \beta^\P \sqrt d \log\br{\tau}}{\sqrt{\tau}}
		\lesssim
		\frac{H^3 d^{2}\log(d H K)}{\sqrt \tau}.
	\end{align*}
	Combining this with the bound on the $Q$-bonus given by \cref{lem:loss_bonus_bound} and replacing $\mathcal E$ for its definition \cref{eq:main_proof_oK}, we finally get
	\begin{align*}
		\E&\sbr{\text{Regret}}
		\\&\lesssim 
		\frac{H^3 d^{2} K \log(d H K)}{\sqrt \tau}     + \beta ( \sqrt d + \sqrt \epsilon) H K
		+ \frac{\tau H \log A}{\eta} 
		+ \epsilon H^2 K 
		+ \frac{\eta \beta^2 H^3 K}{\gamma}
		+ \eta (1+\sigma) H^3 K
            + \frac{ \tau H^2}{\gamma} 
            \\
            &\lesssim
		\gamma H^3 d^{3/2} K     
            + \sqrt \gamma d H^2  K
            + \frac{d H^2 }{ \gamma^3} 
		+ \gamma d H^4 K
            ,
	\end{align*}
        where the second relation follows from $\sigma=1/4, \beta=2H\sqrt{d \gamma}, \epsilon=1/K, \eta=\gamma/(2H)$, and $\tau\approx d/(\sigma\gamma^2)$.
        Balancing the two middle terms by setting $\gamma=K^{-2/7}$ leads to,
        \begin{align*}
		\E\sbr{\text{Regret}}
		\lesssim
		d H^2 K^{6/7}
            + d^{3/2} H^4 K^{5/7} ,
	\end{align*}
        which concludes the proof.
\end{proof}

\subsection[Thm 2]{\cref{thm:polsbe_sim} proof}

Most of the proof below follows the exact same steps as that of \cref{thm:polsbe}. We avoid repeating arguments that are completely identical, and refer the reader to the proof of \cref{thm:polsbe} for the full details.

\begin{proof}[of \cref{thm:polsbe_sim}]
We shall use the following parameter settings;
        $\eta=\gamma/(2H)$,
        $\sigma = 1/4, \epsilon=K^{-1}$,
        $\beta=2H\sqrt {\gamma d}$,
        $\gamma = \frac{2}{(d K)^{2/3}}$
        ,
 and 
	$
	\beta^\P=10 C_1 H^2d^{3/2}\log\br{
			28 C_1 d \beta K H
	}$, where the constant $C_1$ is that specified by \cref{lem:backup_confidence}.

    Similarly to the beginning of \cref{thm:polsbe} we observe
    that \cref{lem:independence_base}, our parameter choices and the setting of $\tau$ imply the conditions for \cref{lem:mgr} are met, thus
 it follows that for all $h, k$  \cref{eq:mgr_norm,eq:mgr_bias,eq:mgr_variance} hold for $\widehat \Sigma_{\gamma}^+ = \widehat \Sigma_{kh\gamma}^+, \Sigma_\gamma = \Sigma_{kh\gamma}$.
 We note that we use here slightly larger datasets $|\D_h^k| = \tau = d^2 M N $ than needed for \cref{lem:mgr}; this is done in order to obtain sharper bounds for the dynamics estimation which enter later in the proof.
Proceeding, we combine \cref{lem:bias_bound,lem:omd_term_bound}
	and set
	\begin{align}
	\mathcal E \eqq \frac{H \log A}{\eta} 
		+ 4 \epsilon H^2 K 
		+ \frac{8 \eta \beta^2 H^3 K}{\gamma}
		+ 2 \eta (1+\sigma) H^3 K	
	,
	\end{align}
	to obtain
	\begin{align*}
		\E\sbr{\tBone + \tBtwo + \tOMD}
		% &\leq 
		% \beta
		% \E\sbr{\sum_{k=1}^K \sum_{h=1}^H
		% 	\E_{s \sim d_h^\star}\sbr{
		% 	\sum_a \br{\pi_h^k(a|s) + \pi^\star_h(a|s)}\norm{\phi(s, a)}_{\widehat \Sigma_{kh\gamma}^+}}
		% }
		% + \mathcal E
		% \\
            &\leq 
            \E\sbr{\sum_{k=1}^K \sum_{h=1}^H
			\E_{(s,a) \sim d_h^\star}\sbr{
			b_h^k(s, a)
			}
		}
            + \mathcal E
            .
	\end{align*}
	This last argument followed in exactly the same manner as in the proof of \cref{thm:polsbe}, with the only difference being the improved bound of the $\tOMD$ term $H\log A/\eta$, that does not have $\tau$ in the numerator (and without the extra $\tau H/\gamma$ term introduced by the last block).
        
Next, again in the same manner of \cref{thm:polsbe},
we claim our choice of parameters are such that conditions of \cref{lem:backup_confidence} are satisfied (in particular, we have for all $h, k$; $|\D_h^k| = \widetilde O((d H K)^4)$) with a $\delta>0$ sufficiently small so that \cref{lem:exploration_2}
        gives;
	\begin{align}
		\E\sbr{\tExploration}
		\leq 
		2\E\sbr{\sum_k \sum_{h=1}^H \E_{s, a \sim d_h^k}\sbr{
				b_h^{\P, k}(s, a)
				+ b_h^k(s, a)
			} }
		- \E\sbr{\sum_{k=1}^K \sum_{h=1}^H 
			\E_{s, a \sim d_h^\star}\sbr{b_h^k(s, a)}
			}
		+ 1
		.
	\end{align}
 
        Adding together our two bounds on the regret terms, we get
	\begin{align*}
		\E\sbr{\text{Regret}}
		&\leq 
		2\E\sbr{\sum_k \sum_{h=1}^H \E_{s, a \sim d_h^k}\sbr{
				b_h^{\P, k}(s, a)
				+ b_h^k(s, a)
			} }
		+ \mathcal E  + 1
		.
	\end{align*}
	Bounding the first term using \cref{lem:loss_bonus_bound,lem:dynamics_bonus_bound}, and replacing $\mathcal E$ for its definition, 
    leads to
	\begin{align*}
		\E&\sbr{\text{Regret}}
		\\&\leq
		40 H\beta^\P \frac{\sqrt d \log\br{2 |\D_h^k|}}{|\D_h^k|} K +
		5 H \sqrt d \beta K
		+ \frac{H \log A}{\eta} 
		+ 4 \epsilon H^2 K 
		+ \frac{8\eta \beta^2 H^3 K}{\gamma}
		+ 2 \eta (\sigma+1) H^3 K
		+ 1
		\\
            &\lesssim
		H^3 d^{3/2} \frac
        {\sqrt d }
        {\sqrt{d^3/\gamma^2} } K +
		H \sqrt d \beta K
		+ \frac{H \log A}{\eta} 
		+ \frac{\eta \beta^2 H^3 K}{\gamma}
		+ \eta H^3 K
            \tag{$\sigma, \epsilon, |D_h^k|, \beta^\P$}
		\\
		&\lesssim 
		\sqrt \gamma d H^2 K 
		+ \frac{H^2}{\gamma}
		+ \gamma d H^4 K
		+ \gamma H^2 K,
            \tag{$\eta, \beta$}
            \\
		&\lesssim
             H^2 (d K)^{2/3}
            + H^4 (dK)^{1/3}
            \tag{$\gamma$}
            .
	\end{align*}
        In the second relation above, we replace $\sigma=1/4, \epsilon=1/K, |D_h^k|=\widetilde \Theta(d^3/\gamma^2), \beta^\P = \widetilde O(H^2 d^{3/2})$, and in the third $\eta=\gamma/(2H), \beta=2H \sqrt {\gamma d}$, simplify and absorb the first term $\gamma \sqrt d H^3 K$ in the $\gamma d H^4 K$ term. Finally, 
        we replace
        $\gamma = \frac{2}{(d K)^{2/3}}$, which completes the proof.
\end{proof}

\section{Regret Terms Proofs}

\subsection{Bias and OMD Terms}
\label{sec:proofs:bias_omd}

\begin{proof}[of \cref{lem:bias_bound}]
    Recall the low dimensional representation $\vq_h^k \in \R^d$ of $Q_h^{k}=Q_h^{k, \pi^k}$ defined in \cref{eq:q_vec_def}
    , and note that
    \begin{align*}
    	\E_k\sbr{\sum_{t=h}^H \l_{t}^k}
    	=
    	\E_k \sbr{Q_h^k(s_h^k, a_h^k)}
    	=
    	\E_k\sbr{\phi(s_h^k, a_h^k)\T \vq_h^k}.
    \end{align*}
    Therefore, 
	\begin{align*}
	    \E_k \sbr{ \widehat \vq_h^k }
	    &= \E_k \sbr{ \widehat \Sigma^+_{kh\gamma} \phi(s_h^k, a_h^k) \phi(s_h^k, a_h^k)\T \vq_h^k }
	    \\
	    % 2
	    &= \E_k \sbr{ \widehat \Sigma^+_{kh\gamma} } 
	    	\Sigma_{kh} \vq_h^k 
	    \tag{independence, \cref{lem:independence_base}}
	    \\
	    % 3
	    &= \Sigma^{-1}_{kh\gamma} \Sigma_{kh} \vq_h^k 
	    + \br{\E_k \sbr{\widehat \Sigma^+_{kh\gamma}} - \Sigma^{-1}_{kh\gamma}}
	    	\Sigma_{kh} \vq_h^k
	    \\
            &=
		\vq_h^k 
		- \gamma \Sigma^{-1}_{kh\gamma} \vq_h^k
  + \br{\E_k \sbr{\widehat \Sigma^+_{kh\gamma}} - \Sigma^{-1}_{kh\gamma}}
	    	\Sigma_{kh} \vq_h^k
      ,
	\end{align*}
	% Now, note that
	% \begin{align*}
	% 	\Sigma^{-1}_{kh\gamma} \Sigma_{kh} \vq_h^k 
	% 	=
	% 	\Sigma^{-1}_{kh\gamma} \br{\Sigma_{kh} + \gamma I} \vq_h^k 
	% 	- \gamma \Sigma^{-1}_{kh\gamma} \vq_h^k 
	% 	=
	% 	\vq_h^k 
	% 	- \gamma \Sigma^{-1}_{kh\gamma} \vq_h^k,
	% \end{align*}
	so for any $s, a$; 
	\begin{align*}
		\E_k \sbr{\phi(s, a)\T \widehat \vq_h^k }
		= \phi(s, a)\T \vq_h^k 
		- \gamma \phi(s, a)\T \Sigma^{-1}_{kh\gamma} \vq_h^k
		+ \phi(s, a)\T \br{\E \sbr{\widehat \Sigma^+_{kh\gamma}} - \Sigma^{-1}_{kh\gamma}}
	    	\Sigma_{kh} \vq_h^k.
	\end{align*}
	To bound the contribution of the third term above, observe that;	\begin{align*}
		\norm{\phi(s, a)\T \br{\E \sbr{\widehat \Sigma^+_{kh\gamma}} - \Sigma^{-1}_{kh\gamma}}
	    	\Sigma_{kh} \vq_h^k
	    }
	    &\leq 
	    \norm{\E \sbr{\widehat \Sigma^+_{kh\gamma}} - \Sigma^{-1}_{kh\gamma}
	    }
	    	\norm{\Sigma_{kh} \vq_h^k
	    }
	    \\
	    &\leq 
	    \epsilon \norm{\Sigma_{kh} \vq_h^k
	    }
	    \\
	    &= 
	    \epsilon \norm{\E_{d_h^k}\sbr{\phi(s_h^k, a_h^k) \phi(s_h^k, a_h^k)\T \vq_h^k}
	    }
	    \\
	    &\leq \epsilon H 
	    	\E_{d_h^k}\sbr{\norm{\phi(s_h^k, a_h^k)}}
	    \\
	    &\leq \epsilon H.
	\end{align*}
	Therefore, using \cref{lem:q_vec};
	\begin{align*}
		\E_k &\sbr{Q_h^k (s, a) - \widehat Q_h^k(s, a)}
		\\
  &= \E_k \sbr{\phi(s, a)\T \br{\vq_h^k - \widehat \vq_h^k}}
		\\
		&\leq \gamma \phi(s, a)\T \Sigma^{-1}_{kh\gamma} \vq_h^k 
		+ \epsilon H
		\\
            % 2
		&= 
		\gamma \phi(s, a)\T \E_k \sbr{\widehat \Sigma^{+}_{kh\gamma}} \vq_h^k 
		+ \gamma \phi(s, a)\T 
			\br{\Sigma^{-1}_{kh\gamma} - \E \sbr{\widehat \Sigma^{+}_{kh\gamma}}} \vq_h^k  
		+ \epsilon H
		\\
            % 3
		&\leq 
		\gamma \phi(s, a)\T \E_k \sbr{\widehat \Sigma^{+}_{kh\gamma}} \vq_h^k 
		+ \gamma \epsilon \sqrt d H  + \epsilon H
            \tag{$\norm{\vq_h^k} \leq H\sqrt d$}
		\\
            % 4
		&\leq 
		\gamma \E_k \sbr{ \phi(s, a)\T \widehat \Sigma^{+}_{kh\gamma} \vq_h^k }
		+ 2 \epsilon H
		\tag{$\gamma \leq 1/\sqrt d$}
		\\
		&\leq 
		\gamma \E_k \sbr{ 
			\norm{\phi(s, a)}_{\widehat \Sigma^{+}_{kh\gamma}} 
			\norm{\vq_h^k}_{\widehat \Sigma^{+}_{kh\gamma}} 
		}
		+ 2 \epsilon H
		\\
		&\leq 
		\sqrt {\gamma d} H \E_k \sbr{ 
			\norm{\phi(s, a)}_{\widehat \Sigma^{+}_{kh\gamma}} 
		}
		+ 2 \epsilon H
            \tag{$\norm[b]{\widehat \Sigma_{kh\gamma}^+} \leq 1/\gamma, \norm{\vq_h^k}\leq H\sqrt d$}
		.
	\end{align*}
	Now, 
	\begin{align*}
		\E &\sbr{ \sum_{k=1}^K \sum_{h=1}^H \E_{s \sim d_h^\star}\sbr{
			\sum_a \pi^k_h(a|s)
			\br{Q_h^k(s, a) - \widehat Q_h^k(s, a)
			} 	
		}}
		\\
		&=
		\E \sbr{ \sum_{k=1}^K \sum_{h=1}^H \E_{s \sim d_h^\star}\sbr{
			\sum_a \pi^k_h(a|s)
			\E_k \sbr{Q_h^k(s, a) - \widehat Q_h^k(s, a) } 	
		}}
		\\
		&\leq 
		\sqrt {\gamma d} H \E \sbr{ \sum_{k=1}^K \sum_{h=1}^H \E_{s \sim d_h^\star}\sbr{
			\sum_a \pi^k_h(a|s)
			\norm{\phi(s, a)}_{\widehat \Sigma^{+}_{kh\gamma}} 
		}}
		+ 2 \epsilon H^2 K.	
	\end{align*}
	The argument for \tBtwo~is identical, apart from summing in the last step over probabilities given by $\pi^\star_h(a|s)$.
	The result follows by summing the two bounds.
\end{proof}

\begin{lemma}[OMD term bound base]
\label{lem:omd_term_bound_base}
		Assume that $\eta \leq \gamma /(2H)$, $\beta \leq 1/2\sqrt \gamma$, and $\gamma \leq 1$. 
  Further, assume that for all $k, h$;
   $\E\sbr{\widehat \Sigma^+_{kh\gamma} \Sigma_{kh\gamma} \widehat \Sigma^+_{kh\gamma}}
		\preceq
		2\E \sbr{\widehat \Sigma^+_{kh\gamma}}
		+ \sigma I$, and
 $\norm[b]{\widehat \Sigma_{kh\gamma}^+} \leq 1/\gamma$ almost surely.
Then, for both \cref{alg:polsbe,alg:polsbe_simulator}, it holds that;
	\begin{align}
        \forall s, a; \quad
		&\av{\widehat Q_h^k(s, a) - \widetilde B_h^k(s, a)} \leq \frac{2 H}{\gamma}
        \label{eq:loss_mag_bound}
        \\
        \forall s, h; \quad
        & \E \sbr{\sum_{k=1}^K \sum_a \pi_h^k(a|s)
	\br{\widehat Q_h^k(s, a) - \widetilde B_h^k(s, a)}^2
		}
		\leq
            \nonumber
            \\
  &\qquad\qquad \frac{2 H^2}{\sqrt \gamma}
		\E \sbr{\sum_{k=1}^K \sum_a \pi_h^k(a|s)
			\norm{\phi(s, a)}_{\widehat \Sigma^+_{kh\gamma}
			}
		}
		+ 2 (\sigma+1) H^2 K 
            + \frac{8 \beta^2 H^2 K}{\gamma}
            \label{eq:2nd_moment_bound}
            .
	\end{align}
\end{lemma}
\begin{proof}
	Note that for any $s, a$, by definition, we have
    $\widehat Q_h^k(s, a) = \phi(s, a)\T \widehat \vq_h^k$ and
    $\av{\widetilde B_h^k(s, a)} \leq B_1^{\rm max}$ (by the clipping in \cref{eq:B_tilde_def}). Thus;
	\begin{align*}
		\av{\widehat Q_h^k(s, a) - \widetilde B_h^k(s, a)}
        \leq 
        \norm{\phi(s, a)\T \widehat \vq_h^k}
        +
        B_1^{\rm max}
	&\leq 
		\norm{\widehat \Sigma_{kh\gamma}^+ \phi(s_h^k, a_h^k)\sum_{t=h}^H \l_t^k}
		+ \frac{2 \beta H}{\sqrt \gamma}
        \\
        &\leq 
		H \norm{\widehat \Sigma_{kh\gamma}^+}
		+ \frac{2 \beta H}{\sqrt \gamma}
	\\
        &\leq H \br{ \frac{1}{\gamma} + \frac{2\beta }{\sqrt \gamma}}
	\leq \frac{2 H}{\gamma},
	\end{align*}
	where the second to last and last inequalities follow from our assumptions 
 $\norm[b]{\widehat \Sigma_{kh\gamma}^+} \leq 1/\gamma$ and 
 $\beta \leq 1/(2\sqrt \gamma)$.
 For the second part, observe that for all $s, h$;
	\begin{align}
		\E &\sbr{\sum_{k=1}^K \sum_a \pi_h^k(a|s)
	\br{\widehat Q_h^k(s, a) - \widetilde B_h^k(s, a)}^2
		}
		\nonumber \\
            &\leq 
			2 \E \sbr{\sum_{k=1}^K \sum_a \pi_h^k(a|s)
				\widehat Q_h^k(s, a)^2
			}
			+ 2 \E \sbr{\sum_{k=1}^K \sum_a \pi_h^k(a|s)
				\widetilde B_h^k(s, a)^2
			}
		\nonumber
		\\
		&\leq 
			2 \E \sbr{\sum_{k=1}^K \sum_a \pi_h^k(a|s)
				\widehat Q_h^k(s, a)^2
			}
			+ \frac{8 \beta^2 H^2 K}{\gamma},
		\label{eq:2ndmoment_1}
	\end{align}
	where the last transition uses again our bound on $\widetilde B_h^k(s, a)$.
	Further, for any $s, a, h, k$, using independence of $\widehat \Sigma_{kh\gamma}^+$ and $(s_h^k, a_h^k, \l_h^k)_{h=1}^H$ conditioned on $\pi^1, \ldots, \pi^k$ (\cref{lem:independence_base}), we have;
	\begin{align}
		\E_k &\sbr{\widehat Q_h^k(s, a)^2}
		\nonumber \\
            &= 
		\E_k \sbr{\phi(s, a)\T \widehat \vq_h^k \br{\widehat \vq_h^k}\T \phi(s, a)}
            \nonumber
		\\
            % 1
		&=
		\E_k \sbr{\phi(s, a)\T 
			\br{\widehat \Sigma^+_{kh\gamma} \phi(s_h^k, a_h^k) L_h^k} 
			\br{\widehat \Sigma^+_{kh\gamma} \phi(s_h^k, a_h^k) L_h^k}\T 
		\phi(s, a)}
		\tag{$L_h^k \eqq \sum_{t=h}^H \l_t^k$}
		\\
            % 2
		&=
		\E_k \sbr{\br{L_h^k}^2\phi(s, a)\T 
			\br{\widehat \Sigma^+_{kh\gamma} \phi(s_h^k, a_h^k)  
			\phi(s_h^k, a_h^k)\T \widehat \Sigma^+_{kh\gamma} }
		\phi(s, a)}
		\nonumber
		\\
            % 3
		&\leq
		H^2 \E_k \sbr{\phi(s, a)\T 
			\widehat \Sigma^+_{kh\gamma} \phi(s_h^k, a_h^k)  
			\phi(s_h^k, a_h^k)\T \widehat \Sigma^+_{kh\gamma} 
		\phi(s, a)}
		\nonumber
		\\
            % 4
		&=
		H^2 \E_k \sbr{\phi(s, a)\T 
			\widehat \Sigma^+_{kh\gamma} 
			\E_{(s_h^k, a_h^k) \sim {\rm Alg}}\sbr{\phi(s_h^k, a_h^k) \phi(s_h^k, a_h^k)\T}\widehat \Sigma^+_{kh\gamma} 
		\phi(s, a)}
            \tag{independence}
		\nonumber
		\\
            % 5
		&=
		H^2 \E_k \sbr{\phi(s, a)\T 
			\widehat \Sigma^+_{kh\gamma} 
			\Sigma_{kh}
			\widehat \Sigma^+_{kh\gamma} 
		\phi(s, a)}
            \nonumber
            \\
            % 6
            &\leq
		H^2 \E_k \sbr{\phi(s, a)\T 
			\widehat \Sigma^+_{kh\gamma} 
			\br{\gamma I + \Sigma_{kh}}
			\widehat \Sigma^+_{kh\gamma} 
		\phi(s, a)}
            \nonumber
            \\
            % 7
            &\leq
		2 H^2 \E_k \sbr{\phi(s, a)\T 
			\widehat \Sigma^+_{kh\gamma} 
		\phi(s, a)}
            + \sigma H^2 
            \tag{$\E\sbr{\widehat \Sigma^+_{kh\gamma} \Sigma_{kh\gamma} \widehat \Sigma^+_{kh\gamma}}
		\preceq
		2\E \sbr{\widehat \Sigma^+_{kh\gamma}}
		+ \sigma I$}
            \nonumber
            \\
            % 7
            &\leq
		\frac{2 H^2}{\sqrt \gamma} \E_k \sbr{
		\norm{\phi(s, a)}_{\widehat \Sigma^+_{kh\gamma}}
		}
            + \sigma H^2 
            \tag{$\norm{\phi(s, a)}_{\widehat \Sigma_{kh\gamma}^+}\leq \frac{1}{\sqrt\gamma}$}
            % \label{eq:omdterm_Q_1}
            .
	\end{align}
        Now,
	\begin{align*}
		2\E \sbr{\sum_{k=1}^K \sum_a \pi_h^k(a|s)
				\widehat Q_h^k(s, a)^2
			}
		\leq
		\frac{2 H^2}{\sqrt \gamma}
		\E \sbr{\sum_{k=1}^K \sum_a \pi_h^k(a|s)
			\norm{\phi(s, a)}_{\widehat \Sigma^+_{kh\gamma}
			}
		}
		+ 2 \sigma H^2 K 
            + 2 H^2 K,
	\end{align*}
	and the result follows by plugging the above back into \cref{eq:2ndmoment_1}.
\end{proof}

\begin{lemma}[\cref{alg:polsbe_simulator} OMD term bound]
\label{lem:omd_term_bound}
		Assume that \cref{alg:polsbe_simulator} is executed
   with $\eta \leq \gamma /(2H)$, $\beta \leq 1/2\sqrt \gamma$, and $\gamma \leq 1$. 
  Further, assume that for all $k, h$;
   $\E\sbr{\widehat \Sigma^+_{kh\gamma} \Sigma_{kh\gamma} \widehat \Sigma^+_{kh\gamma}}
		\preceq
		2\E \sbr{\widehat \Sigma^+_{kh\gamma}}
		+ \sigma I$, and
 $\norm[b]{\widehat \Sigma_{kh\gamma}^+} \leq 1/\gamma$ almost surely.
Then for any $s, h$, we have;
	\begin{align*}
		\E &\sbr{ \sum_{k=1}^K 
			\ab{\widehat Q_h^k(s, \cdot) - \widetilde B_h^k(s, \cdot), 
				\pi^k_h(\cdot|s) - \pi^\star_h(\cdot|s)}
		}
		\\
		&\leq
		\frac{\log A}{\eta}
		+ \frac{8\eta \beta^2 H^2 K}{\gamma}
		+ \frac{2\eta H^2}{\sqrt \gamma} \E \sbr{
			\sum_{k=1}^K \sum_a \pi_h^k(a|s)\norm{\phi(s, a)}_{\widehat \Sigma_{kh\gamma}^+}
		}
		+ 2 \eta H^2 K (1 + \sigma)
		.
	\end{align*}
\end{lemma}
\begin{proof}
    By \cref{eq:loss_mag_bound} of \cref{lem:omd_term_bound_base}
    and our condition of $\eta \leq \frac{\gamma}{2H}$, we may apply the OMD bound \cref{lem:omd}, 
	which gives for all $s, h$;
	\begin{align*}
		\E &\sbr{ \sum_{k=1}^K 
			\ab{\widehat Q_h^k(s, \cdot) - \widetilde B_h^k(s, \cdot), 
				\pi^k_h(\cdot|s) - \pi^\star_h(\cdot|s)}
		}
		\leq \frac{\log A}{\eta} 
			+ \eta \E \sbr{\sum_{k=1}^K \sum_a \pi_h^k(a|s)
	\br{\widehat Q_h^k(s, a) - \widetilde B_h^k(s, a)}^2
		}.
	\end{align*}
    The result now follows by bounding the secon term above with \cref{eq:2nd_moment_bound} given by \cref{lem:omd_term_bound_base}.
\end{proof}

Next, we give the proof of \cref{lem:omd_term_bound_blocking} that combines the blocking OMD regret bound \cref{lem:blocking_omd} with \cref{lem:omd_term_bound_base}.
\begin{proof}[of \cref{lem:omd_term_bound_blocking}]
    By \cref{eq:loss_mag_bound} of \cref{lem:omd_term_bound_base} and our assumption that $\eta \leq \frac{\gamma}{2H}$, the conditions for blocking OMD regret bound \cref{lem:blocking_omd} are met. 
	Thus, for all $s, h$;
	\begin{align*}
		\E &\sbr{ \sum_{k=1}^K 
			\ab{\widehat Q_h^k(s, \cdot) - \widetilde B_h^k(s, \cdot), 
				\pi^k_h(\cdot|s) - \pi^\star_h(\cdot|s)}
		}
		\\
		&\leq \frac{\tau \log A}{\eta} 
                + \frac{2 \tau H}{\gamma}
			+ \eta \E \sbr{\sum_{k=1}^K \sum_a \pi_h^k(a|s)
	\br{\widehat Q_h^k(s, a) - \widetilde B_h^k(s, a)}^2
		}.
	\end{align*}
    The result now follows by bounding the second term above with \cref{eq:2nd_moment_bound} given by \cref{lem:omd_term_bound_base}.
\end{proof}

\subsection{Exploration Terms}
\label{sec:proofs:exploration}

\begin{proof}[of \cref{lem:exploration_2}]
By our assumption and \cref{lem:exploration_1}, the random variable
\begin{align*}
		Z 
		&\eqq
		-\sum_{k=1}^K \sum_{h=1}^H \E_{s \sim d_h^\star}\sbr{
			\ab{\widetilde B_h^k(s, a), 
				\pi^k_h(\cdot|s) - \pi^\star_h(\cdot|s)}
		}
		+
		2\sum_k \sum_{h=1}^H \E_{s, a \sim d_h^k}\sbr{
				b_h^{\P, k}(s, a)
				+ b_h^k(s, a)
			} 
		\\
		&\quad - \sum_{k=1}^K \sum_{h=1}^H 
			\E_{s, a \sim d_h^\star}\sbr{b_h^k(s, a)}
		.
\end{align*}
is non-negative w.p. $\geq 1-\delta$. In addition, it is not hard to verify that
\begin{align*}
	|Z| \leq  
	2 K H (2\beta H/ \sqrt \gamma) 
	+ 2 K H (\beta/\sqrt \gamma + \beta^\P/\sqrt \lambda)
	+ K H \beta/\sqrt \gamma
	\leq 7 K H^2 (\beta/\sqrt \gamma + \beta^\P/\sqrt \lambda)
 \leq \delta^{-1}.
\end{align*}

Thus, $Z$ is supported on $[-D, D]$ for $D\eqq7 K H^2 (\beta/\sqrt \gamma + \beta^\P/\sqrt \lambda)$, which implies
\begin{align*}
	\E Z 
        \geq - \delta D
	= -\delta 7 K H^2 (\beta/\sqrt \gamma + \beta^\P/\sqrt \lambda)
	\geq -1,
\end{align*}
which completes the proof after rearranging the terms.
\end{proof}

The next lemma is partially implicit in \citet{luo2021policy} Lemma B.1, but extends it to incorporate the affect of the bonus-to-go approximations. In addition, we provide a simpler argument owed to the removal of the dilation term, and by letting the extended value difference \cref{lem:extended_value_diff} handle most of the technicalities.
\begin{lemma}
\label{lem:exploration_1}
	Assume that the approximate bonus-to-go functions $\widetilde B_{h+1}^k \colon \S \times \A \to \R$ computed by the algorithm satisfy for all $s, a, h, k$;
	\begin{align*}
	b_h^k(s, a) + \P_h\widetilde W_{h+1}^k(s, a)
		\leq
		\widetilde B_h^k(s, a) 
		\leq b_h^k(s, a) + \P_h\widetilde W_{h+1}^k(s, a)
		+ 2 b_h^{\P, k}(s, a)
	\end{align*}
	Then the exploration term is bounded as
	\begin{align*}
		\sum_{k=1}^K &\sum_{h=1}^H \E_{s \sim d_h^\star}\sbr{
			\ab{\widetilde B_h^k(s, a), 
				\pi^k_h(\cdot|s) - \pi^\star_h(\cdot|s)}
		}
		\\
		&\leq 
		2\sum_{k=1}^K \sum_{h=1}^H \E_{s, a \sim d_h^k}\sbr{
				b_h^{\P, k}(s, a)
				+ b_h^k(s, a)
			} 
		- \sum_{k=1}^K \sum_{h=1}^H 
			\E_{s, a \sim d_h^\star}\sbr{b_h^k(s, a)}
		.
	\end{align*}
\end{lemma}
\begin{proof}
	By assumption, for any $s,a,h,k$;
 $
     0 \leq \widetilde B_h^k(s, a) 
     - b_h^k(s, a) - \P_h\widetilde W_{h+1}^k(s, a),
 $
 thus,
	\begin{align*}
		\sum_{k=1}^K \sum_{h=1}^H \E_{s \sim d_h^\star}\sbr{
			\ab{\widetilde B_h^k(s, \cdot), 
				\pi^k_h(\cdot|s) - \pi^\star_h(\cdot|s)}
		}
            &\leq
		\sum_{k=1}^K \sum_{h=1}^H \E_{s \sim d_h^\star}\sbr{
			\ab{\widetilde B_h^k(s, a), 
				\pi^k_h(\cdot|s) - \pi^\star_h(\cdot|s)}
		}
		\\
            &\quad +
		\sum_{k=1}^K \sum_{h=1}^H \E_{s, a \sim d_h^\star}\sbr{
			\widetilde B_h^k(s, a) - b_h^k(s, a) - \P_h\widetilde W_{h+1}^k(s, a)
		}
		\\
		&= \sum_{k=1}^K \widetilde W_1^k - W_1^{k, \pi^\star}
		,
	\end{align*}
 where the equality follows from the extended value difference \cref{lem:extended_value_diff} with $\widehat V^\pi_1 = \widetilde W_1^k$ and $V_1^{\pi'} = W_1^{k, \pi^\star}$ (and we recall definitions in \cref{eq:W_tilde_def,eq:W_true_def}).
	Further, again by \cref{lem:extended_value_diff} and our upper bound on $\widetilde B_h^k$;
	\begin{align*}
		\sum_{k=1}^K \widetilde W^k_1 - W^{k, \pi^k}_1
		&=
		\sum_{k=1}^K \sum_{h=1}^H \E_{s, a \sim d_h^k}\sbr{
				 \widetilde B_h^k(s, a) 
     - b_h^k(s, a) - \P_h\widetilde W_{h+1}^k(s, a)
			} 
		\\
		&\leq
			2\sum_{k=1}^K \sum_{h=1}^H \E_{s, a \sim d_h^k}\sbr{
				b_h^{\P, k}(s, a)
			}
		.
	\end{align*}
	In addition, by definition of the true bonus value functions,
	\begin{align*}
		\sum_{k=1}^K W_1^{k, \pi^k} - W_1^{k, \pi^\star}
		=
		\sum_{k=1}^K \sum_{h=1}^H 
			\E_{s, a \sim d_h^k}\sbr{b_h^k(s, a)}
		- \sum_{k=1}^K \sum_{h=1}^H 
			\E_{s, a \sim d_h^\star}\sbr{b_h^k(s, a)}
        ,
	\end{align*}
        thus we see that,
        \begin{align*}
            \sum_{k=1}^K \widetilde W_1^k - W_1^{k, \pi^\star}
            &=  
            \sum_{k=1}^K \widetilde W_1^k - W_1^{k, \pi^k}
            + \sum_{k=1}^K W_1^{k, \pi^k}
                - W_1^{k, \pi^\star}
            \\
            &\leq 
            2\sum_{k=1}^K \sum_{h=1}^H \E_{s, a \sim d_h^k}\sbr{
				b_h^{\P, k}(s, a)
			}
            + \sum_{k=1}^K \sum_{h=1}^H 
			\E_{s, a \sim d_h^k}\sbr{b_h^k(s, a)}
		- \sum_{k=1}^K \sum_{h=1}^H 
			\E_{s, a \sim d_h^\star}\sbr{b_h^k(s, a)}
            \\
		&\leq 
		2\sum_{k=1}^K \sum_{h=1}^H \E_{s, a \sim d_h^k}\sbr{
				b_h^{\P, k}(s, a)
				+ b_h^k(s, a)
			} 
		- \sum_{k=1}^K \sum_{h=1}^H 
			\E_{s, a \sim d_h^\star}\sbr{b_h^k(s, a)},
        \end{align*}
        which completes the proof.
\end{proof}

\subsection{Bonus Terms}

\begin{lemma}
\label{lem:dynamics_bonus_bound}
	The dynamics bonus functions $b_h^{\P, k}$ samples in $\D_h^k$, satisfy for all episodes $k$ and all time steps $h$;
	\begin{align*}
		\E \sbr{\E_{s, a \sim d_h^k}\sbr{ b_h^{\P, k}(s, a) } }
		\leq  \frac{10 \beta^\P \sqrt d \log\br{2 |\D_h^k|}}{\sqrt{|\D_h^k|}}.
	\end{align*}
\end{lemma}
\begin{proof}
	Follows immediately by \cref{lem:dynamics_bonus_bound_base} with $\delta=|\D_h^k|^{-2}$, and noting that 
	$|b_h^{\P, k}(s, a)| \leq \beta^\P$ almost surely.
\end{proof}

\begin{lemma}
\label{lem:dynamics_bonus_bound_base}
	Assume $\lambda \geq 1$, and let $h, k$.
	For all $\delta > 0$, we have that the following holds w.p.~$\geq 1-\delta$:
	\begin{align*}
		\E_{s, a \sim d_h^k}\sbr{ b_h^{\P, k}(s, a) } 
		\leq  \frac{5\beta^{\P} \sqrt {d}  \log\br{2 |\D_h^k|/\delta}}{\sqrt{|\D_h^k|}}.
	\end{align*}
\end{lemma}
\begin{proof}
	Let $N \eqq |\D_h^k|$, and observe;
	\begin{align}
		\E_{s, a \sim d_h^k}\sbr{ b_h^{\P, k}(s, a) } 
		&= \beta^\P 
		\E_{\tilde s_h, \tilde a_h \sim d_h^k}
        \sbr{ \norm{\phi(\tilde s, \tilde a)}_{\br{\Lambda_h^k}^{-1}}}
		\nonumber 
		\\
		&=
		\frac{\beta^\P}{N}
		\E_{(\tilde s_h^1, \tilde a_h^1), 
            \ldots, (\tilde s_h^N, \tilde a_h^N)  
            \sim d_h^k}
            \sbr{ \sum_{i=1}^N \norm{\phi(\tilde s_h^i, \tilde a_h^i)}_{\br{\Lambda_h^k}^{-1}}}
		\label{eq:dynamics_bonus_bound_base_1}.
	\end{align}
	Further, let $\Lambda_h^{k,i} = \lambda I + \sum_{t=1}^{i-1} \phi(s_h^t, a_h^t) \phi(s_h^t, a_h^t)\T $ for some arbitrary ordering $(s_h^i, a_h^i)_{i=1}^N$ of the elements in $\D_h^k$. Then,
	\begin{align*}
		\E_{d_h^k}\sbr{ \sum_{i=1}^N \norm{\phi(\tilde s_h^i, \tilde a_h^i)}_{\br{\Lambda_h^k}^{-1}}}
		\leq 
		\E_{d_h^k}\sbr{ \sum_{i=1}^N \norm{\phi(\tilde s_h^i, \tilde a_h^i)}_{\br{\Lambda_h^{k, i}}^{-1}}}
		.
	\end{align*}
	Now, by \cref{lem:nameless_concentration} with 
	$X_i \eqq \norm{\phi(\tilde s_h^i, \tilde a_h^i)}_{\br{\Lambda_h^{k, i}}^{-1}}$;
	\begin{align*}
		\E_{ d_h^k}\sbr{ \sum_{i=1}^N \norm{\phi(\tilde s_h^i, \tilde a_h^i)}_{\br{\Lambda_h^{k, i}}^{-1}}}
		&\leq 2\sum_{i=1}^N
			\norm{\phi(s_h^i, a_h^i)}_{\br{\Lambda_h^{k, i}}^{-1}}
		+ \frac{4}{\sqrt \lambda} \log\frac{2 N}{\delta}
		\\
		&\leq 2\sum_{i=1}^N
			\norm{\phi(s_h^i, a_h^i)}_{\br{\Lambda_h^{k, i}}^{-1}}
		+ 4 \log\frac{2 N}{\delta}
		\\
		&\leq 2\sqrt{N \sum_{i=1}^N
			\norm{\phi(s_h^i, a_h^i)}_{\br{\Lambda_h^{k, i}}^{-1}}^2}
		+ 4 \log\frac{2 N}{\delta}
		.
	\end{align*}
	By \cref{lem:eliptical_potential}, we can further bound this by
	\begin{align*}
		2\sqrt{2 N d\log \br{ 1+ \frac{N}{d \lambda } }}
		+ 4 \log\frac{2 N }{\delta}
		\leq 5\sqrt {N d} \log\frac{2 N}{\delta}.
	\end{align*}
	Combining the derived inequality with \cref{eq:dynamics_bonus_bound_base_1}, we get
	\begin{align*}
		\E_{s, a \sim d_h^k}\sbr{ b_h^{\P, k}(s, a) } 
		\leq  \frac{5 \beta^\P \sqrt {d} \log\br{2 N/\delta}}{\sqrt{N}},
	\end{align*}
	which completes the proof.
\end{proof}

\begin{lemma}
\label{lem:loss_bonus_bound}
Assuming $\norm{ \E_k\sbr{\widehat \Sigma_{kh\gamma}^+} - \Sigma_{kh\gamma}^{-1}} \leq \epsilon$,
	it holds that
	\begin{align*}
		\E_k \sbr{ \E_{s, a \sim d_h^k}\sbr{ b_h^k(s, a) }}
		\leq 2\beta (\sqrt d + \sqrt \epsilon)
	\end{align*}
\end{lemma}
\begin{proof} Note that
	\begin{align}
		\E_{s, a \sim d_h^k}\sbr{ b_h^k(s, a) }
		&= \beta  
		\E_{s, a \sim d_h^k}\sbr{ 
			\norm{\phi(s, a)}_{\widehat \Sigma_{kh\gamma}^+}
			+ \sum_a \pi_h^k(a'|s) 
			\norm{\phi(s, a')}_{\widehat \Sigma_{kh\gamma}^+}
		}
		\nonumber
		\\
		&=
		2 \beta  
		\E_{s, a \sim d_h^k}\sbr{ 
			\norm{\phi(s, a)}_{\widehat \Sigma_{kh\gamma}^+}
		}
		\label{eq:lbb_1}
	\end{align}
	Further, for any $s, a$,
	\begin{align*}
		\E_k & \sbr{\norm{\phi(s, a)}_{\widehat \Sigma^+_{kh\gamma}}}
		\\
		% 1
		&= \E_k \sbr{
			\sqrt{ \phi(s, a)\T \Sigma^{-1}_{kh\gamma} \phi(s, a)
			+ \phi(s, a)\T \br{ \widehat \Sigma^+_{kh\gamma} -\Sigma^{-1}_{kh\gamma} }\phi(s, a)
			}
		}
		\\
		% 2
		&\leq 
			\sqrt{ \phi(s, a)\T \Sigma^{-1}_{kh\gamma} \phi(s, a)
			+ \phi(s, a)\T \br{ \E_k \sbr{ \widehat \Sigma^+_{kh\gamma} } -\Sigma^{-1}_{kh\gamma} }\phi(s, a)
			}
		\tag{Jensen's inequality}
		\\
		% 3
		&\leq 
			\sqrt{ \phi(s, a)\T \Sigma^{-1}_{kh\gamma} \phi(s, a) }
			+ 
			\sqrt{ \phi(s, a)\T \br{ \E_k \sbr{ \widehat \Sigma^+_{kh\gamma} } -\Sigma^{-1}_{kh\gamma} }\phi(s, a)
			}
		\\
		% 4
		&\leq 
			\sqrt{ \phi(s, a)\T \Sigma^{-1}_{kh\gamma} \phi(s, a) }
			+ 
			\sqrt{ \norm { \E_k \sbr{ \widehat \Sigma^+_{kh\gamma} } -\Sigma^{-1}_{kh\gamma} }_{\rm op}
			}
		\\
		% 5
		&\leq 
			\sqrt{ \phi(s, a)\T \Sigma^{-1}_{kh\gamma} \phi(s, a) }
			+ 
			\sqrt{ \epsilon }.
\end{align*}
Now, conditioning on $\pi^k$, we have;
\begin{align*}
	\E_k \sbr{ \E_{s, a \sim d_h^k}\sbr{ b_h^k(s, a) }}
	&=
	2 \beta  
		\E_k \sbr{\E_{s, a \sim d_h^k}\sbr{ 
			\norm{\phi(s, a)}_{\widehat \Sigma_{kh\gamma}^+}
		}}
	\tag{\cref{eq:lbb_1}}
	\\
	&=
	2 \beta  
		\E_{s, a \sim d_h^k}\sbr{ \E_k \sbr{
			\norm{\phi(s, a)}_{\widehat \Sigma_{kh\gamma}^+}
		}}
	\tag{$d_h^k \perp \widehat \Sigma^+_{kh\gamma} \mid \pi^k $}
	\\
	% 3
	&\leq
	2 \beta  
		\E_{s, a \sim d_h^k}\sbr{ 
		\sqrt{ \phi(s, a)\T \Sigma^{-1}_{kh\gamma} \phi(s, a) }
		}
		+ 2\beta \sqrt{ \epsilon }
	\tag{previous inequality}
	\\
	% 4
	&\leq
	2 \beta  
		\sqrt{ \E_{s, a \sim d_h^k}\sbr{ \phi(s, a)\T \Sigma^{-1}_{kh\gamma} \phi(s, a) }
		}
		+ 2\beta \sqrt{ \epsilon }
	\tag{Jensen}
	\\
	% 5
	&\leq
	2 \beta  
		\sqrt{ \E_{s, a \sim d_h^k}\sbr{ \phi(s, a)\T \Sigma^{-1}_{kh} \phi(s, a) }
		}
		+ 2\beta \sqrt{ \epsilon }
	\tag{$\Sigma^{-1}_{kh\gamma} \preceq \Sigma^{-1}_{kh}$}
	\\
	% 6
	&=
	2 \beta  
		\sqrt{ \E_{s, a \sim d_h^k}\sbr{ \tr\br{\Sigma^{-1}_{kh} \phi(s, a) \phi(s, a)\T } }
		}
		+ 2\beta \sqrt{ \epsilon }
	\\
	% 7
	&=
	2 \beta  
		\sqrt{ \tr\br{\Sigma^{-1}_{kh} \E_{s, a \sim d_h^k}\sbr{ \phi(s, a) \phi(s, a)\T } }
		}
		+ 2\beta \sqrt{ \epsilon }
	\\
	% 8
	&=
	2 \beta  
		\sqrt{ \tr\br{\Sigma^{-1}_{kh} \Sigma_{kh} }
		}
		+ 2\beta \sqrt{ \epsilon }
	\\
	% 9
	&=
	2 \beta \br{\sqrt d + \sqrt{ \epsilon } },
\end{align*}
which completes the proof.
\end{proof}

\section{Approximate bonus-to-go confidence bounds}
\label{sec:proofs:backup_confidence}
In this section, we establish optimism / bonus-bias confidence bounds on our approximate bonus action-value functions (aka bonus-to-go).
These follow from uniform concentration over the estimated bonus value function backup operator which is computed by the least squares regression procedure in \cref{alg:olspe}. The arguments given here, at a conceptual level, follow those of \citet{jin2020provably}.

\paragraph{Bonus value functions explored by the algorithm.}
Define
\begin{align*}
	B&(s, a; \beta, \Sigma^+, \beta^\P, \Lambda, w, B_{\max},\pi) 
	\\
	&= {\rm clip}\sbr{\beta \br[B]{ 
			\norm{\phi(s, a)}_{ \Sigma^+}
			+ \sum_a \pi(a'|s) 
			\norm{\phi(s, a')}_{ \Sigma^+}
		}
	+\phi(s, a)\T w + \beta^{\P}\norm{\phi(s, a)}_{\Lambda^{-1}}
	}_0^{B_{\max}}
\end{align*}
and
\begin{align}
	\mathcal B&(\beta, \lambda_{\Sigma^+}, \beta^\P, \lambda_\Lambda,
		L, B_{\max}, \pi)
        \label{eq:B_function_class_def}
	\\
	&= \cb{B(s, a; \beta, \Sigma, \beta^\P, \Lambda, w, \pi) 
	\mid 
	\lambda_{\max}(\Sigma^+) \leq \lambda_{\Sigma^+},
	\lambda_{\min}(\Lambda) \geq \lambda_\Lambda,
	\norm{w} \leq L
	}
        \nonumber
	\\
	\mathcal W&(\beta, \lambda_{\Sigma^+}, \beta^\P, \lambda_\Lambda,
		L, B_{\max}, \pi)
        \label{eq:W_function_class_def}
	\\
	&= \cb{W\colon \S \to \R; W(s) = \ab{\pi(\cdot|s), B(s, \cdot)}
		\mid B \in \mathcal B(\beta, \lambda_{\Sigma^+}, \beta^\P, \lambda_\Lambda,
		L, \pi)
	}
        \nonumber
	.
\end{align}
We note that with appropriate parameter choices, $\widetilde B_h^k \in \mathcal B$ and $\widetilde W_h^k \in \mathcal W$ for $\widetilde B_h^k, \widetilde W_h^k$ computed by \cref{alg:olspe} and defined in \cref{eq:B_tilde_def,eq:W_tilde_def}.
This will be made rigorous in the proof of \cref{lem:backup_confidence} below.

\begin{proof}[of \cref{lem:backup_confidence}]
	By \cref{lem:alg_weights_bound}, our choice of $\beta$, $\lambda \geq 1$, and that $|\D_h^k| = \widetilde O((HdK)^4)$, we have $\norm{\widehat \vw_h^k} \leq a H^5 d^4 K^4 \log(d K) $ for some constant $a$.
	Further, again by our choice of $\beta$, $2\beta H/\sqrt \gamma =  4 H^2 \sqrt d$, thus by algorithm definition and our assumptions, it is readily verified that;
	\begin{align*}
		\widetilde W_{h+1}^k 
		\in \mathcal W \eqq \mathcal W(\beta, 1/\gamma, \beta^\P, \lambda,
		\boldsymbol{L} = a H^5 d^4 K^4 \log(d K), 
		\boldsymbol{B_{\max}} = 4 H^2 \sqrt d, 
		\pi^k)
	\end{align*} 
	Hence, by \cref{lem:covering_bonus_value_functions},
		there exist $c > 0$ such that for any $\epsilon>0$,
		\begin{align*}
			\log \mathcal N_{\epsilon}(\mathcal W) 
			\leq c d^2 \log\br{\frac{4\beta^\P\beta H K d}{\gamma \lambda \epsilon}}
                \leq c_{\rm cov} d^2 \log\br{\frac
				{d \beta^\P\beta H K}
				{\epsilon}}
			,
		\end{align*}
	where the second inequality follows from our assumptions that $\gamma \geq 1/K$, $\lambda \geq 1$ and the appropriate choice of constant $c_{\rm cov}$.
	Thus we may apply \cref{lem:dynamics_backup_error_base}, to obtain
	that for the constant $C$ specified by the lemma, with 
	$
		\beta^\P 
		\geq 8 C H^2 d^{3/2} \log\br{
			d \beta K H /\delta
		}
		\geq C (4 H^2\sqrt d) d\log\br{
			d \beta K^2 H /\delta
		}
	$,
	we have w.p.~$\geq 1-\delta$ that for all $s,a, h, k$;
	\begin{align}
	\label{eq:simulator_ols_confidence}
		\av{\phi(s, a)\T \widehat \vw_h^k - \P_h \widetilde W_{h+1}^k(s, a)}
		\leq \beta^\P \norm{\phi(s, a)}_{\br{\Lambda_h^k}^{-1}}
		= b_h^{\P, k}(s, a)
		.
	\end{align} 
	This establishes that
	$
		0 
		\leq (\widetilde \P_h^k - \P_h) \widetilde W_{h+1}^k (s, a) 
		\leq 2 b_h^{\P, k}(s, a)
	$ holds for all $s,a, h, k$, leaving us only with the task to verify the truncations defined in \cref{eq:B_tilde_def} do not interfere with the desired conclusion.
	First, we show that;
	\begin{align}
		\widetilde B_h^k(s, a) 
		&\leq b_h^k(s, a) + \P_h\widetilde W_{h+1}^k(s, a)
		+ 2 b_h^{\P, k}(s, a).
		\label{eq:bc_bias}
	\end{align}
	Indeed, by definition \cref{eq:B_tilde_def};
	\begin{align*}
		\widetilde B_h^k(s, a) 
		&= {\rm clip}\sbr{b_h^k(s, a) 
		+ \widetilde \P_h^k 
		\widetilde W_{h+1}^k(s, a)}_{0}^{2\beta (H-h+1)/\sqrt \gamma}
		\\
		&\leq {\rm clip}\sbr{b_h^k(s, a) 
		+ \widetilde \P_h^k \widetilde W_{h+1}^k(s, a)}_{0}^{\infty},
	\end{align*}
	and when $\widetilde B_h^k(s, a) = 0$, \cref{eq:bc_bias} holds trivially as all RHS terms are non-negative.
	Otherwise,
	\begin{align*}
		\widetilde B_h^k(s, a) 
		\leq b_h^k(s, a) 
		+ \widetilde \P_h^k \widetilde W_{h+1}^k(s, a)
		&=
		b_h^k(s, a) 
		+ \phi(s, a)\T \widehat \vw_h^k + b_h^{\P, k}(s, a)
		\tag{def. in \cref{eq:backup_dynamics_def}}
		\\
		&\leq b_h^k(s, a) + \P_h\widetilde W_{h+1}^k(s, a)
		+ 2 b_h^{\P, k}(s, a).
		\tag{\cref{eq:simulator_ols_confidence}}
	\end{align*}
	Next, to verify 
	\begin{align}
		\widetilde B_h^k(s, a)  
		&\geq b_h^k(s, a) + \P_h\widetilde W_{h+1}^k(s, a),
		\label{eq:bc_optimism}
	\end{align}
	note that
	\begin{align*}
		b_h^k(s, a) + \P_h\widetilde W_{h+1}^k(s, a) 
		\leq 
		\frac{2\beta}{\sqrt \gamma} 
		+ \frac{2\beta(H - h)}{\sqrt\gamma}
		= \frac{2\beta(H - h + 1)}{\sqrt\gamma}.
	\end{align*}
	Thus, when $\widetilde B_h^k(s, a) = 2\beta (H-h+1)/\sqrt\gamma$, \cref{eq:bc_optimism} holds trivially.
	Otherwise,
	\begin{align*}
		\widetilde B_h^k(s, a) 
		&= {\rm clip}\sbr{b_h^k(s, a) 
		+ \widetilde \P_h^k 
		\widetilde W_{h+1}^k(s, a)}_{0}^{2\beta (H-h+1)/\sqrt \gamma}
		\\
		&= {\rm clip}\sbr{b_h^k(s, a) 
		+ \widetilde \P_h^k \widetilde W_{h+1}^k(s, a)}_{0}^{\infty}
		\\
		&\geq b_h^k(s, a) 
		+ \widetilde \P_h^k \widetilde W_{h+1}^k(s, a)
		\\
		&\geq b_h^k(s, a) 
		+ \phi(s, a)\T \widehat \vw_h^k + b_h^{\P, k}(s, a)
		\\
		&\geq b_h^k(s, a) 
		+ \P_h \widetilde W_{h+1}^k(s, a)
		\tag{\cref{eq:simulator_ols_confidence}}
		,
	\end{align*}
	which completes the proof.
\end{proof}

\begin{lemma}[Approximate backup operator error bound]
\label{lem:dynamics_backup_error_base}
	Let $\D_h^k$ be the dataset used for episode $k$ of size $\widetilde O((dHK)^4)$, and 
		$\Lambda_h^k = \lambda I + \sum_{i\in \D_h^k}\phi(s_h^i, a_h^i)\phi(s_h^i, a_h^i)\T $, with $\lambda \geq 1$.
		Further, let $\mathcal V$ be a function class with $\log \mathcal N_{\epsilon} (\mathcal V) \leq c_{\rm cov}d^2\log\br[b]{\frac{d \beta \beta^\P K}{ \epsilon}}$ for any $\epsilon > 0$, and $\norm{f}_\infty \leq B_{\max}$ for all $f \in \mathcal V$.
	Then there exists a constant $C > 0$ depending only on $c_{\rm cov}$, such that letting
	\begin{align*}
		\beta^\P \geq C B_{\max} d \log \br{
			\frac{d \beta K H}{ \delta}
		},
	\end{align*} 
	ensures that with probability $\geq 1 - \delta$ it holds that
	for all $f \in \mathcal V$	and all $s, a, h, k$;
	\begin{align*}
		\av{\phi(s, a)\T \widehat w_f - \P_h f(s, a)}
		\leq \beta^\P \norm{\phi(s, a)}_{\br{\Lambda_h^k}^{-1}},
	\end{align*} 
	where $\widehat w_f = \br{\Lambda_h^k}^{-1} \sum_{i\in \D_h^k}\phi(s_h^i, a_h^i) f(s_{h+1}^i)$.
\end{lemma}
\begin{proof}
	Fix $k, h$, and
	define $w_f^\star$ by
	\begin{align*}
		\P_h f(s, a) = \phi(s, a)\T \int \psi_h(s')f(s'){\rm d}s'
		:= \phi(s, a)\T w_f^\star.
	\end{align*}
	Note that by normalization assumptions in \cref{def:linmdp}, we have that $\norm{w_f^\star} \leq \sqrt d B_{\max}$,
	thus, by \cref{lem:ols_error};
	\begin{align}
		\norm{\widehat w_f - w_f^\star}_{\Lambda_h^k}
		&\leq 
		\norm{\sum_{i \in \D_h^k} \phi(s_h^i, a_h^i)
			\br{ f(s_{h+1}^i) 
			- \phi(s_h^i, a_h^i)\T w_f^\star}
		}_{\br{\Lambda_h^k}^{-1}}
		+ \sqrt {\lambda d} B_{\max}
		\label{eq:backup_error_1}.
	\end{align}
	In addition, by \cref{lem:selfnorm_unif_concentration}, we have that w.p.~$\geq 1-p$;
	\begin{align*}
		&\norm{\sum_{i \in \D_h^k} \phi(s_h^i, a_h^i)
			\br{ f(s_{h+1}^i) 
			- \phi(s_h^i, a_h^i)\T w_f^\star}
		}_{\br{\Lambda_h^k}^{-1}}^2
		\\
		&\leq 4 B_{\max}^2\br{
			\frac{d}{2} \log\br{\frac{|\D_h^k| + \lambda}{\lambda}} 
			+ \log \frac{\mathcal N_\epscov(\mathcal V)}{p}
		}	
		+ \frac{8 |\D_h^k|^2 \epsilon^2}{\lambda},
		\\
		&\leq 2 B_{\max}^2 d \log\br{\frac{|\D_h^k| + \lambda}{\lambda}} 
		+ 4 c_{\rm cov} B_{\max}^2d^2\log\br{\frac{d \beta \beta^\P K}{ \epscov p}}
		+ \frac{8 |\D_h^k|^2 \epsilon^2}{\lambda}
		\\
		&\leq c (\epscov |\D_h^k| B_{\max} d)^2\log\br{\frac{d \beta \beta^\P K}{ \epscov p}},
	\end{align*}
	for some constant $c \geq 1$ that depends only on $c_{\rm cov}$. 
	Now, using that $|\D_h^k| = \widetilde O((d H K )^4)$,  with an appropriate choice of $\epsilon_{\rm cov}$ and we can further bound the last display by 
	\begin{align*}
		c' (B_{\max} d)^2\log\br{\frac{d \beta \beta^\P K}{p}},
	\end{align*}
	where $c'$ is another constant $ \geq 1$. 
	Combining this with \cref{eq:backup_error_1}, we get that w.p. $1-p$;
	\begin{align*}
		\norm{\widehat w_f - w_f^\star}_{\Lambda_h^k}
		&\leq 
		2 c'  B_{\max} d \sqrt {\log\br{\frac{d \beta \beta^\P K}{  p}}
		}
		.
	\end{align*}
	By the union bound over $k, h$, choosing $\delta = p/(KH)$, we have that w.p.~$1-\delta$, it holds that for all $k, h$;
	\begin{align*}
		\norm{\widehat w_f - w_f^\star}_{\Lambda_h^k}
		&\leq 
		4 c' B_{\max} d \sqrt {\log\br{
			\frac{d \beta \beta^\P L K H}{ \delta}
		}}
		.
	\end{align*}
	Now, by \cref{lem:beta_recursion_choice}, setting
	\begin{align*}
		\beta^\P = 8 c' B_{\max} d \log \br{
			\frac{d \beta K H}{ \delta}
		}
		\geq 
		4 c' B_{\max} d \log \br{
			\frac{d \beta  K H}{ \delta}
			\times 4 c' B_{\max} d
		}
	\end{align*}
	ensures that 
	$	\norm{\widehat w_f - w_f^\star}_{\Lambda_h^k}
		\leq
		\beta^\P.
	$
	Finally, observe that for all $s, a$;
	\begin{align*}
		\av{\phi(s, a)\T \widehat w_f - \P_h f(s, a)}
		&=
		\av{\phi(s, a)\T\br{ \widehat w_f - w_f^\star}}
		\\
            &\leq 
		\norm{\phi(s, a)}_{\br{\Lambda_h^k}^{-1}}
		\norm{\widehat w_f - w_f^\star}_{\Lambda_h^k}
		\leq \beta^\P \norm{\phi(s, a)}_{\br{\Lambda_h^k}^{-1}},
	\end{align*}
	which complete the proof.
\end{proof}

\begin{lemma}
\label{lem:beta_recursion_choice}
 	Let $R, z \geq 1$, and $x \geq 2 z \log (R z)$. Then
	$z \log (R x) \leq x$.
\end{lemma}
\begin{proof}
	If $x = 2 z \log (R z)$;
	\begin{align*}
		  z \log (R x)
		&=z\log R + z\log (2 z \log (R z))
		\\
		&=z\log R + z\log (2 z) +z\log  \log (R z)
		\\
		&\leq z \log R + z \log z + z \log (R z)
		\\
		&= 2 z \log R + 2 z \log z 
		\\
		&= x.
	\end{align*}
	For larger values, the result follows by noting 
	$x- z\sqrt {\log (R x)}  $ is monotonically increasing in $x$ for all $x \geq z$.
\end{proof}

The next lemma bounds the norm of the weights $\widehat \vw_h^k$ computed in the OLSPE algorithm. 
We note a tighter bound can be shown, as in \citet{jin2020provably} Lemma B.2, but the simpler argument below is sufficient for our purposes.
\begin{lemma}\label{lem:alg_weights_bound}
	For all $k\in[K], h\in[H]$, assuming running {\rm OLSPE} (\cref{alg:olspe}) with dataset $\D_h^k$, we have $\norm{\widehat \vw_h^k} \leq 2 \beta H |\D_h^k|  / \sqrt{\gamma \lambda}$.
\end{lemma}
\begin{proof}
	We have;
	\begin{align*}
		\norm{\widehat \vw_h^k}
		&= \norm[B]{\br{\Lambda_h^k}^{-1} \sum_{i\in \D^k_h} \phi(s_h^i, a_h^i)\widetilde W_{h+1}^k(s_{h+1}^i)}
		\\
		&\leq 
		\br{2\beta H/\sqrt \gamma} 
			\norm{\br{\Lambda_h^k}^{-1}} 
			\norm[B]{\sum_{s_h, a_h\in \D^k_h} \phi(s_h, a_h)}
		\leq\frac{ 2\beta H |\D_h^k|}{\sqrt {\gamma \lambda}}
	,
	\end{align*}
	where the first inequality follows from  $\norm[b]{\widetilde W_{h+1}^k}_\infty 
	\leq \norm[b]{\widetilde B_{h+1}^k}_\infty \leq 2\beta H/\sqrt \gamma$, as per \cref{eq:B_tilde_def}.
\end{proof}

\begin{lemma}
\label{lem:ols_error}
	Let $\{\phi_i\}_{i=1}^n \in \R^d, \{y_i\}_{i=1}^n \in \R$, $\lambda \in \R$, and set $\Lambda \eqq \sum_{i=1}^N \phi_i\phi_i\T + \lambda I$, and
	$\widehat w = \Lambda^{-1} \sum_{i=1}^N \phi_i y_i$. Then
	\begin{align*}
		\norm{\widehat w - w^\star}_{\Lambda}
		&\leq
		\norm{\sum_{i=1}^N \phi_i \br{
			y_i - \phi\T w^\star
		}}_{\Lambda^{-1}}
		+ 
		\sqrt \lambda \norm{w^\star}
	\end{align*}
\end{lemma}
\begin{proof}
	We have
	\begin{align*}
		\widehat w - w^\star 
		= \Lambda^{-1} \sum_{i=1}^N \phi_i y_i
		- \Lambda^{-1}\br{\sum_{i=1}^N \phi_i\phi_i\T + \lambda I}w^\star
		= \Lambda^{-1} \sum_{i=1}^N \phi_i \br{
			y_i - \phi\T w^\star
		}
		+ \lambda \Lambda^{-1}w^\star,
	\end{align*}
	which implies 
	\begin{align*}
		\norm{\widehat w - w^\star}_{\Lambda}
		\leq \norm{\sum_{i=1}^N \phi_i \br{
			y_i - \phi\T w^\star
		}}_{\Lambda^{-1}}
		+ 
		\lambda \norm{w^\star}_{\Lambda^{-1}}
		\leq 
		\norm{\sum_{i=1}^N \phi_i \br{
			y_i - \phi\T w^\star
		}}_{\Lambda^{-1}}
		+ 
		\sqrt \lambda \norm{w^\star},
	\end{align*}
	as required.
\end{proof}

\subsection{Uniform concentration for bonus value functions}
In this section we provide lemmas that support uniform concentration over bonus value functions explored by the algorithm. The bound on the covering number of the euclidean ball stated below is standard.

\begin{lemma}[Covering number of Euclidean Ball]
\label{lem:covering_l2_ball}
	For any $\epsilon > 0$, the $\epsilon$-covering of the Euclidean ball in $\R^d$ with radius $R > 0$ is upper bounded by $(1+2R/\epsilon)^d$.
\end{lemma}

The next lemma follows from (relatively standard) arguments that are essentially the same as those of Lemma D.6 in \citet{jin2020provably}.
\begin{lemma}
\label{lem:covering_bonus_value_functions}
	Let $\mathcal N_\epsilon(\F)$ denote the $\norm{\cdot}_\infty$ covering number of a function class $\F$.
	For some universal constant $c > 0$, we have
	\begin{align*}
		\log \mathcal N_\epsilon (\mathcal W(\beta, \lambda_{\Sigma^+}, \beta^\P, \lambda_\Lambda,
		L, B_{\max}, \pi)) \leq 
		c d^2 \log\br{\frac{d \beta^\P\beta \lambda_{\Sigma^+} L}{\lambda_\Lambda  \epsilon}},
	\end{align*}
    for the function class $\mathcal W$ as defined in \cref{eq:W_function_class_def}.
\end{lemma}
\begin{proof}
	First, we remove clipping (that can only decrease the covering number),
	and reparameterize the $\mathcal B$ function class \cref{eq:B_function_class_def} with $A=\br{\beta^{\P}}^2\Lambda^{-1}$ and $E = \beta^2 \Sigma^+$, to consider functions of the form
	\begin{align*}
		B(s, a; E, A, w)
		= 
		\norm{\phi(s, a)}_{E}
			+ \sum_a \pi(a'|s) 
			\norm{\phi(s, a')}_{ E}
		+\phi(s, a)\T w + \norm{\phi(s, a)}_{A},
	\end{align*}
	with parameters $\norm{w}\leq L, \norm{A} \leq \br{\beta^\P}^2 \lambda_\Lambda^{-1}$, and $\norm{E} \leq \beta^2 \lambda_{\Sigma^+}$.
	Recall that $\norm{\phi(s, a)} \leq 1$, and observe,
	\begin{align*}
		&\av{ 
			B(s, a; E_1, A_1, w_1)  
			-
			B(s, a; E_2, A_2, w_2)
		}
		\\
		&\leq 
		\av{
			\sqrt{\phi(s, a)\T E_1 \phi(s, a)}
			-
			\sqrt{\phi(s, a)\T E_2 \phi(s, a)}
		}
		\\ &\quad + 
		\sum_{a'} \pi(a'|s) \av{
			\sqrt{\phi(s, a')\T E_1 \phi(s, a')}
			-
			\sqrt{\phi(s, a')\T E_2 \phi(s, a')}
		}
		\\
		&\quad + 
		\norm{\phi(s, a)} \norm{w_1 - w_2}
		+
		\av{
			\sqrt{\phi(s, a)\T A_1 \phi(s, a)}
			-
			\sqrt{\phi(s, a)\T A_2 \phi(s, a)}
		}
		\\
		&\leq 
		\sqrt{\av{\phi(s, a)\T (E_1 - E_2 )\phi(s, a)}}
		+ 
		\sum_{a'} \pi(a'|s) 
			\sqrt{\av{\phi(s, a')\T (E_1 - E_2 )\phi(s, a')}}
		\\
		&\quad + 
		\norm{w_1 - w_2}
		+
		\sqrt{\av{\phi(s, a)\T \br{A_1 - A_2 }\phi(s, a)}}
		\\
		&\leq 
		2\sqrt{\norm{E_1 - E_2}}
		+ \norm{w_1 - w_2}
		+ \sqrt{\norm{A_1 - A_2}}
		\\
		&\leq 
		2\sqrt{\norm{E_1 - E_2}_F}
		+ \norm{w_1 - w_2}
		+ \sqrt{\norm{A_1 - A_2}_F}
	\end{align*}
	Now, we consider an $\epsilon^2/16$ net over 
	$\cb[b]{E\subset \R^{d \times d} \mid \norm{E}_F \leq \sqrt d \beta^2 \lambda_{\Sigma^+}}$,
	an $\epsilon/2$ net over $\cb{w\in \R^d \mid \norm{r}\leq L}$, and 
	an $\epsilon^2/4$ net over $\cb[b]{A \subset \R^{d \times d} \mid \norm{A}_F \leq \sqrt d \br{\beta^\P}^2 \lambda_{\Lambda}^{-1}}$. 
	Noting that for any matrix $M$, $\norm{M}_F \leq \sqrt d \norm{M}$, we have that the product of these three nets provides an $\epsilon$-net over the original parameter space.
	By \cref{lem:covering_l2_ball}, this implies
	\begin{align*}
		\log \mathcal N_\epsilon (\mathcal B)
		\leq 
		d \log(1 + 4L/\epsilon) 
		+ d^2 \log\br{1 + 8\sqrt d \br{\beta^\P}^2 \lambda_{\Lambda}^{-1} \epsilon^{-2}}
		+ d^2 \log\br{1 + 8\sqrt d \beta^2 \lambda_{\Sigma^+} \epsilon^{-2}}.
	\end{align*}
	Finally, noting that $\pi$ is a parameter that is held fixed, and that $W(s)$ just averages over values of $B(s, \cdot)$, we have $\log \mathcal N_\epsilon (\mathcal W) \leq \log \mathcal N_\epsilon (\mathcal B)$, and the result follows.
\end{proof}

The next lemma is brought as is from \citet{jin2020provably}, except from slight adaptation of notation.
We remark that due to the blocking structure / simulator in our algorithms, we could in fact use a similar weaker version of this lemma suitable for random design least squares regression, rather than the one below which is suitable for a martingale setting.

\begin{lemma}[\textbf{Uniform concentration of self normalized processes}; \citet{jin2020provably} Lemma D.4]
	\label{lem:selfnorm_unif_concentration}
	Let $\cb{x_\tau}$ be a stochastic process on state space $\S$ with corresponding filtration $\cb{\F_\tau}_{\tau=1}^\infty$.
	Let $\cb{\phi_\tau}$ be an $\R^d$-valued stochastic process where $\phi_\tau \in \F_\tau$, and $\norm{\phi_\tau} \leq 1$.
	Further, let $\Lambda_n = \lambda I + \sum_{\tau=1}^n \phi_\tau \phi_\tau\T$.
Then for any $\delta > 0$, with probability at least $1-\delta$, for all $n\geq 1$ and any $V \in \mathcal V$ so that $\norm{V}_\infty \leq D$,
we have;
	\begin{align*}
		\norm{\sum_{\tau=1}^n \phi_\tau 
			\br[B]{V(x_\tau) - \E \sbr{ V(x_\tau)|\F_{\tau-1} }}}_{\Lambda_n^{-1}}^2
		\leq 4 D^2\br{
			\frac{d}{2} \log\br{\frac{n + \lambda}{\lambda}} 
			+ \log \frac{\mathcal N_\epsilon(\mathcal V)}{\delta}
		}	
		+ \frac{8 n^2 \epsilon^2}{\lambda},
	\end{align*}
	where $\mathcal N_\epsilon(\mathcal V) $ is $\norm{\cdot}_\infty$ covering number of $\mathcal V$.
\end{lemma}

\section{Matrix Geometric Resampling Lemma Proof}
\label{sec:proofs:mgr}
As mentioned, our \cref{alg:mgr} is similar to that of \citet{luo2021policy}, which itself is the original proposed by \citet{neu2020efficient} (see also \citealp{neu2021online, neu2020online_v1}), but with averaging over multiple estimators.
We present here a different analysis to obtain tighter bounds in the 2'nd moment term analysis given in \cref{lem:omd_term_bound}.

\begin{proof}[of \cref{lem:mgr}]
	First, note that since $\gamma < 1/2$ and $c=1/2$;
	\begin{align*}
		\norm{\widehat \Sigma^{(n)}_{m, \gamma}} \leq (1-c\gamma)^n
		&\implies 
		\norm{\widehat \Sigma^{+}_{m, \gamma}} \leq c \sum_{n=0}^N (1-c\gamma)^n
		\leq \frac1\gamma
		\\
		&\implies 
		\norm{\widehat \Sigma^{+}_{\gamma}} 
		\leq \frac{1}{ \gamma}.
	\end{align*}
	For the bias claim,
	using independence of samples;
	\begin{align*}
		\E \widehat \Sigma^{(n)}_{m, \gamma}
		&= \prod_{i=1}^n(I - c \E\sbr{\gamma I + \phi_{m, i} \phi_{m, i}\T} )
		= \prod_{i=1}^n(I - c \Sigma_\gamma )
		= (I - c \Sigma_\gamma )^n
		\\
		\implies
		\E \widehat \Sigma^{+}_{m, \gamma} 
		&= c I + c \sum_{n=1}^N (I - c \Sigma_\gamma )^n
		= c \sum_{n=0}^N (I - c \Sigma_\gamma )^n
		,
	\end{align*}
	hence, 
	\begin{align*}
		\E \widehat \Sigma^{+}_\gamma 
		= c \sum_{n=0}^N (I - c \Sigma_\gamma )^n
		= \Sigma_\gamma^{-1} 
		- \sum_{n=N+1}^\infty (I - c \Sigma_\gamma )^n,
	\end{align*}
	where we use that $\gamma < 1/2$ and $c=1/2$ imply all eigenvalues of $I - c \Sigma_\gamma$ are in $(0, 1)$, and $A^{-1} = \sum_{n=0}^\infty (I-A)^n$ for any invertible matrix $A$ with all eigenvalues $\in (0, 1)$.
	Now,
	\begin{align*}
		\norm{\E\sbr{\widehat\Sigma_{\gamma}^+} - \Sigma_{\gamma}^{-1} }_{\rm op}
		\leq 
		\norm{(I - c \Sigma_{\gamma})^{N+1}}_{\rm op}
		\norm{\Sigma_{h\gamma}^{-1}}_{\rm op}
		\leq 
		\br{1-c \gamma}^N\frac{1}{\gamma}
		\leq 
		e^{-c\gamma N}\frac1\gamma
		=\epsilon,
	\end{align*}
	where in the last step we substitute $c=1/2$ and $N= \frac{2}{\gamma}\log\frac{1}{\gamma \epsilon}$.
	
	Now for the last claim,
	note that for any $m$, $\widehat \Sigma^{+}_{m, \gamma}$ is a sum of positive definite matrices, with the first term being $c I$, 
	thus $\lambda_{\min}\br{\widehat \Sigma^{+}_m} \geq 1/2$.
	In addition, by \cref{lem:beta_recursion_choice},
	\begin{align*}
		M = \frac{48 d}{\gamma \sigma}\log\frac{72 d}{\gamma^2 \sigma}
		\geq 
		\frac{12 d}{\gamma (\sigma/2)}\log \frac{3M}{\gamma}
		\implies
		\sigma/2 \geq
		\frac{12 d}{\gamma M}\log \frac{3M}{\gamma}
		,
	\end{align*}
	therefore our assumption that $\sigma \leq 1/4$ verifies the conditions for \cref{lem:Sigma3_bernstein} are met. Thus, we obtain;
	\begin{align*}
	\E\sbr{\widehat \Sigma^+_\gamma \Sigma_\gamma \widehat \Sigma^+_\gamma}
		&\preceq
		2\E \sbr{\widehat \Sigma^+_\gamma}
		+ \br{3\epsilon + \frac{12 d}{\gamma M}\log \frac{3M}{\gamma}
		} I
		,
	\end{align*}
	and by the previous display,
%	\begin{align*}
%		M = \frac{48 d}{\gamma \sigma}\log\frac{72 d}{\gamma^2 \sigma}
%		\geq 
%		\frac{12 d}{\gamma (\sigma/2)}\log \frac{3M}{\gamma}.
%	\end{align*}
%	Hence
	\begin{align*}
	\E\sbr{\widehat \Sigma^+_\gamma \Sigma_\gamma \widehat \Sigma^+_\gamma}
		&\preceq
		2\E \sbr{\widehat \Sigma^+_\gamma}
		+ \br{3\epsilon + \sigma/2} I
		.
	\end{align*}
	The proof is complete by our assumption that $\epsilon\leq\sigma/6$.
\end{proof}

\begin{lemma}
\label{lem:Sigma3_bernstein}
	Let $0 < \epsilon < 1/16, 0 < \gamma < 1/2$, and
	assume $\widehat \Sigma^+_1, \ldots, \widehat \Sigma^+_M \in \R^{d\times d}$ are $M$ i.i.d.~random matrices and $\Sigma \in \R^{d\times d}$ is a fixed matrix such that 
		$\gamma \preceq \Sigma \preceq I$,  and $\norm[b]{\E\sbr[b]{\widehat \Sigma^+} - \Sigma^{-1}}
		\leq \epsilon$ where
		$\widehat \Sigma^+ \eqq \frac1M \sum_{m=1}^M \widehat \Sigma^+_m $.
		Further, assume that
		$(1/2) I \preceq \widehat \Sigma^+_m \preceq (1/\gamma) I$ almost surely for all $m$, and 
		$\frac{8 d}{\gamma M}\log \frac{3 M}{\gamma} < 1/8$.
		Then, 
	\begin{align*}
		\E\sbr{\widehat \Sigma^+ \Sigma \widehat \Sigma^+}
		&\preceq
		2 \E\sbr{\widehat \Sigma^+}
		+ \br{3\epsilon + \frac{12 d}{\gamma M}\log \frac{3M}{\gamma} 
		} I
	\end{align*}
\end{lemma}
\begin{proof}
	Denote $\Sigma^+ \eqq \E \sbr{\widehat \Sigma^+}$. 
	By assumption,
	\begin{align*}
		\Sigma^+
		= 
		\Sigma^{-1} 
		+ \br{ \Sigma^+ - \Sigma^{-1}}
		\preceq 
		\Sigma^{-1} + \epsilon I
		,
	\end{align*}
	thus by \cref{lem:Sigma_bernstein},
	\begin{align}
		\widehat \Sigma^+ 
		\preceq 2 \Sigma^+ + \alpha I
		&\preceq
		2 \Sigma^{-1} + (2\epsilon + \alpha) I
		\nonumber
		\\
		\iff 
		\widehat \Sigma^+ - (2\epsilon + \alpha) I
		&\preceq
		2 \Sigma^{-1}
		\label{eq:Sig_minus_sig_ineq}
	\end{align}
	holds with probability $\geq 1-\delta$ and $\alpha \eqq \frac{4 d}{\gamma M}\log \frac{3 M}{\delta}$.
	Now, as long as $\alpha' \eqq 2\epsilon + \alpha < 1/4$, we have that 
	\begin{align*}
		\lambda_{\min}\br{\widehat \Sigma^{+} - \alpha' I}
		\geq
		1/2 - \alpha'
		\geq 1/4,
	\end{align*}
	therefore the matrices on both sides of \cref{eq:Sig_minus_sig_ineq} are positive definite, hence
	\begin{align*}
		\Sigma
		\preceq
		2\br{\widehat \Sigma^{+} - \alpha' I}^{-1}.
	\end{align*}
	This implies that,
	\begin{align*}
		\widehat \Sigma^+ \Sigma \widehat \Sigma^+
		=
		\br{\widehat \Sigma^+ - \alpha' I}  \Sigma \widehat \Sigma^+
		+ \alpha'  \Sigma \widehat \Sigma^+
		\preceq
		2 \widehat \Sigma^+ 
		+ \alpha'  \Sigma \widehat \Sigma^+,
	\end{align*}
	holds w.p.~$\geq 1-\delta$.
	This, and considering that 
	$\norm{\widehat \Sigma^+ \Sigma \widehat \Sigma^+ - 2\Sigma^+ - \alpha' \Sigma \widehat \Sigma^+} \leq \frac{1}{\gamma^2} + \frac2\gamma +\frac1\gamma \leq \frac4{\gamma^2}$,
	implies that for any $\delta > 0$;
	\begin{align*}
		\E\sbr{\widehat \Sigma^+ \Sigma \widehat \Sigma^+}
		&\preceq
		2\E\sbr{\widehat \Sigma^+}
		+ \alpha' \Sigma \E\sbr{ \widehat \Sigma^+}
		+ \frac{4 \delta}{\gamma^2} I
		\\
		&=
		2\Sigma^+
		+ \alpha' \Sigma \Sigma^+
		+ \frac{4 \delta}{\gamma^2} I
		\\
		&=
		2\Sigma^+
		+ \alpha' I + \alpha' \Sigma\br{\Sigma^+ - \Sigma^{-1} }
		+ \frac{4 \delta}{\gamma^2} I
		\\
		&\preceq
		2\Sigma^+
		+ \alpha' I 
		+ \epsilon I
		+ \frac{4 \delta}{\gamma^2} I
		\\
		&\preceq
		2\Sigma^+
		+ \br{3\epsilon + \frac{4 d}{\gamma M}\log \frac{3M}{\delta} 
		+ \frac{4 \delta}{\gamma^2}} I
		,
%		\label{eq:sigma3_bound}
	\end{align*}
	with the last equality following simply by plugging in the definition of $\alpha'$.
	Choosing $\delta = \gamma/M$, 
	we may now see that by our assumptions,
	\begin{align*}
		\alpha' 
		\eqq 2\epsilon + \alpha 
		= 
		2\epsilon +
		\frac{4 d}{\gamma M}\log \frac{3 M^2}{\gamma}
		\leq 
            \frac{1}{8}
            + \frac{8 d}{\gamma M}\log \frac{3M}{\gamma}
		< 1/4,
	\end{align*}
	which verifies our earlier requirement on $\alpha'$.
	The proof is complete by plugging our choice of $\delta$ into the previous display.
    % \begin{align*}
    %     \frac{4 d}{\gamma M}\log \frac{3M^2}{\gamma} + 
    %     \frac{4}{\gamma M}
    % \end{align*}
\end{proof}

\begin{lemma}
\label{lem:Sigma_bernstein}
	Assume $\widehat \Sigma^+_1, \ldots, \widehat \Sigma^+_M \in \R^{d\times d}$ are $M$ i.i.d.~random matrices such that $\norm{\widehat \Sigma^+_m}\leq 1/\gamma$ almost surely and $\E \widehat \Sigma^+_m = \Sigma^+$. Then, for $\widehat \Sigma^+=\frac1M \sum_{i=1}^M \widehat \Sigma_m^+$ and $\alpha =\frac{4 d}{\gamma M}\log \frac{3 M}{\delta}$, we have
	\begin{align*}
		\widehat \Sigma^+ \preceq 2   \Sigma^+ + \alpha I
		.
	\end{align*}	
\end{lemma}
\begin{proof}
	For any fixed $\phi \in \R^d$ with $\norm{\phi}=1$, we have by \cref{lem:nameless_concentration_ub} that w.p. $\geq 1-\delta$:
	\begin{align*}
		\sum_{m=1}^M \phi\T \widehat \Sigma^+_m \phi
		&\leq 2 \sum_{m=1}^M \phi\T \Sigma^+ \phi
		+ \frac{1}{\gamma}\log \frac{1}{\delta}
		\\
		\implies 
		\phi\T \widehat \Sigma^+ \phi
		&\leq 2 \phi\T \Sigma^+ \phi 
		+ \frac{1}{\gamma M}\log \frac{1}{\delta}
		.
	\end{align*}
	Consider now an $\epsilon$-net over the unit sphere in $\R^d$ of size $(1 + 2/\epsilon)^d$, which exists by \cref{lem:covering_l2_ball}. 
	By the union bound we have that w.p.~$1-\delta$, for all $\tilde \phi$ in the net it holds that;
	\begin{align*}
		\tilde \phi\T \widehat \Sigma^+ \tilde \phi
		\leq 2 \tilde \phi\T  \Sigma^+ \tilde \phi 
		+ \frac{d}{\gamma M}\log \frac{3}{\delta \epsilon}
		,
	\end{align*}
	Thus, w.p.~$1-\delta$, for any $\phi \in \R^d, \norm{\phi}= 1$;
	\begin{align*}
		\phi\T \widehat \Sigma^+ \phi
		\leq 2 \phi\T  \Sigma^+ \phi 
		+ \frac{3\epsilon^2}{\gamma}
		+ \frac{d}{\gamma M}\log \frac{3}{\delta \epsilon}
		\leq \frac{4 d}{\gamma M}\log \frac{3 M}{\delta}
		= \alpha,
	\end{align*}
	with the last inequality following from choosing $\epsilon = 1/M$. This implies that
	\begin{align*}
		\forall \phi, \norm{\phi}=1; \quad \phi\T \widehat \Sigma^+ \phi
		&\leq \phi\T \br{2 \Sigma^+ + \alpha I} \phi
		\\
		\implies 
		\forall \phi\in \R^d; \quad \phi\T \widehat \Sigma^+ \phi
		&\leq \phi\T \br{2 \Sigma^+ + \alpha I} \phi,
	\end{align*}
	which completes the proof.
\end{proof}

% \begin{lemma}[Covering number of Euclidean Ball]
% \label{lem:covering_l2_ball}
% 	For any $\epsilon > 0$, the $\epsilon$-covering of the Euclidean ball in $\R^d$ with radius $R > 0$ is upper bounded by $(1+2R/\epsilon)^d$.
% \end{lemma}

\begin{lemma}
\label{lem:nameless_concentration_ub}
	Let $\cb{X_i}_{i=1}^N$ be a sequence of i.i.d.~random variables supported on $[0, B]$. Then with probability $\geq 1-\delta$, we have that;
	\begin{align*}
		\sum_{i=1}^N X_i 
		\leq 2\sum_{i=1}^N \E\sbr{X_i}
		+ B \log \frac{1}{\delta}
		.
	\end{align*}
\end{lemma}
\begin{proof}
	Let $Z_i \eqq X_i/B, \mu_i \eqq \E[Z_i]$, and observe;
	\begin{align*}
		\E \sbr{e^{Z_i}} 
		\leq \E\sbr{1 + Z_i + Z_i^2}
		\leq 1 + 2\mu_i
		\leq e^{2\mu_i}
		,
	\end{align*}
	where the first inequality follows from $e^z \leq 1 + z + z^2$ for $z \in [0, 1]$, and the last from $1 + z \leq e^z$.
	By independence of the $Z_i$, this implies that 
	\begin{align*}
	\E\sbr{e^{\sum_{i=1}^N Z_i - 2\mu_i}} 
	= \prod_{i=1}^N \E\sbr{e^{Z_i - 2\mu_i}} 
	\leq 1,
	\end{align*}
	and therefore by Markov's inequality,
	\begin{align*}
		\Pr\br[B]{ \sum_{i=1}^N Z_i - 2\mu_i \geq w}
		= \Pr\br{ e^{\sum_{i=1}^N Z_i - 2\mu_i} \geq e^{w}}
		\leq \E\sbr{e^{\sum_{i=1}^N Z_i - 2\mu_i}}e^{-w}
		\leq e^{-w}.
	\end{align*}
	Setting $\delta \eqq e^{-w}$, we get that w.p.$\geq 1-\delta$,
	$\sum_{i=1}^N Z_i \leq 2\sum_{i=1}^N \mu_i + \log\frac1\delta$.
	The result follows by substituting $Z_i$ for $X_i/B$ and rearranging.
\end{proof}

\section{Additional Lemmas}

\begin{lemma}[See Lemma D.4 in \citet{rosenberg2020near}]
\label{lem:nameless_concentration}
	Let $(\F_i)_{i=1}^\infty$ be a filtration, and let $(X_i)_{i=1}^\infty$ be a sequence of random variables that are $\F_i$-measurable, and supported on $[0, B]$. Then with probability $\geq 1-\delta$, we have that for any $N \geq 1$;
	\begin{align*}
		\sum_{i=1}^N \E\sbr{X_i \mid \F_{i-1}}
		\leq 2 \sum_{i=1}^N X_i + 4 B \log \frac{2 K}{\delta}.
	\end{align*}
\end{lemma}

\begin{lemma}[Elliptical potential lemma, see \citet{lattimore2020bandit}, Lemma 19.4]
\label{lem:eliptical_potential}
	Let $(\phi_i)_{i=1}^N \subset \R^d$ with $\norm{\phi_i} \leq 1$, and set $\Lambda_i \eqq \lambda I + \sum_{t=1}^{i-1} \phi_t\phi_t\T  $ where $\lambda \geq 1$. Then,
	\begin{align*}
		\sum_{i=1}^N \norm{\phi_i}_{\Lambda_i^{-1}}^2 
		\leq 
		2 d\log \br{ 1+ \frac{N}{d \lambda } }
	\end{align*}
\end{lemma}
\begin{proof}
	Note that $\lambda \geq 1$ implies 
	$\norm{\phi_i}_{\Lambda_i^{-1}}^2 
		\leq \lambda_{\max}(\Lambda_i^{-1}) \norm{\phi_i}^2
		\leq \lambda^{-1} \leq 1$.
	Thus 
	\begin{align*}
		\sum_{i=1}^N \norm{\phi_i}_{\Lambda_i^{-1}}^2 
		= \sum_{i=1}^N \min \cb{1, \norm{\phi_i}_{\Lambda_i^{-1}}^2 }.
	\end{align*}
	The rest of the proof is identical to \citet{lattimore2020bandit}, with $L=1$ and $V_0 = \lambda I$.
\end{proof}

\begin{lemma}[Extended value difference, \citet{shani2020optimistic} Lemma 1, see also \cite{cai2020provably}]
\label{lem:extended_value_diff}
	Let $M = (\S, \A, H, \P, \l)$ be any MDP and $\pi, \pi' \in \S \to \Delta(\A)$ be any two policies.
	Then, for any sequence of functions $\widehat Q_h^\pi \colon \S\times \A \to \R, \widehat V_h^\pi \colon \S \to \R$, where $\widehat V_h^\pi(s) \eqq \ab{\pi_h(\cdot | s), \widehat Q_h^\pi(s, \cdot)}$,  $h = 1, \ldots, H$, we have
\begin{align*}
	\widehat V_1^{\pi} - V_1^{\pi'}
	&= 
	\sum_{h=1}^H \E_{s_h \sim d_h^{\pi'}}\sbr{
		\ab{\widehat Q^\pi_h(s_h, \cdot), 
		\pi_h(\cdot|s_h) - \pi'_h(\cdot|s_h)
		}
	}
	\\
	&\quad+
	\sum_{h=1}^H \E_{s_h, a_h \sim d_h^{\pi'}}\sbr{
		\widehat Q^\pi_h(s_h, a_h) - \l_h(s_h, a_h)
			- \P_h \widehat V_{h+1}^{\pi}(s_h, a_h)
	}.
\end{align*}
% \us{Closer to notation in the context this is used in the analysis:
% \begin{align*}
% 	\widetilde W_1^{\pi} - W_1^{\pi^\star}
% 	&= 
% 	\sum_{h=1}^H \E_{s_h \sim d_h^{\pi^\star}}\sbr{
% 		\ab{\widetilde B^\pi_h(s_h, \cdot), 
% 		\pi_h(\cdot|s_h) - \pi^\star_h(\cdot|s_h)
% 		}
% 	}
% 	\\
% 	&\quad+
% 	\sum_{h=1}^H \E_{s_h, a_h \sim d_h^{\pi^\star}}\sbr{
% 		\widetilde B^\pi_h(s_h, a_h) - \l_h(s_h, a_h)
% 			- \P \widetilde W_{h+1}^{\pi}(s_h, a_h)
% 	}.
% \end{align*}
% }
\end{lemma}
\begin{proof}
	For any $s\in \S, h\in [H]$, we have
	\begin{align*}
		\widehat V_h^\pi(s) - V_h^{\pi'}(s) 
		&= \ab{\pi_h(\cdot | s), \widehat Q_h^\pi(s, \cdot)}
		- \ab{\pi_h'(\cdot | s), Q_h^{\pi'}(s, \cdot)}
		\\
		&= \ab{\pi_h(\cdot | s) - \pi_h'(\cdot | s), \widehat Q_h^\pi(s, \cdot)}
		+ \ab{\pi_h'(\cdot | s), \widehat Q_h^\pi(s, \cdot) - Q_h^{\pi'}(s, \cdot)}
	\end{align*}
	Further, by the Bellman consistency equations, for all $a$;
	$
		Q_h^{\pi'}(s, a) = \l_h(s, a) + \P_h V_{h+1}^{\pi'}(s, a)
	$,
	thus
	\begin{align*}
		\ab{\pi_h'(\cdot | s), \widehat Q_h^\pi(s, \cdot) - Q_h^{\pi'}(s, \cdot)}
		&= \E_{a\sim \pi'(\cdot|s)} \sbr{
			\widehat Q_h^\pi(s, a) 
			- \l_h(s, a) - \P_h V_{h+1}^{\pi'}(s, a)
		}
		\\
		&= \E_{a\sim \pi'(\cdot|s)} \sbr{
			\widehat Q_h^\pi(s, a) 
			- \l_h(s, a) - \P_h \widehat V_{h+1}^{\pi}(s, a)
		} 
		\\
		&\quad+ \E_{a\sim \pi'(\cdot|s)} \sbr{
			\P_h \widehat V_{h+1}^{\pi}(s, a)
			- \P_h V_{h+1}^{\pi'}(s, a)
		}
		\\
		&= \E_{a\sim \pi'(\cdot|s)} \sbr{
			\widehat Q_h^\pi(s, a) 
			- \l_h(s, a) - \P_h \widehat V_{h+1}^{\pi}(s, a)
		} 
		\\
		&\quad+ \E_{s' \sim \P_h(\cdot|s, a), a \sim \pi'(\cdot|s)} \sbr{
			\widehat V_{h+1}^{\pi}(s')
			- V_{h+1}^{\pi'}(s')
		}.
	\end{align*}
	Combining the last two displays we obtain
	\begin{align*}
		\widehat V_h^\pi(s) - V_h^{\pi'}(s)
		&= \ab{\pi_h(\cdot | s) - \pi_h'(\cdot | s), \widehat Q_h^\pi(s, \cdot)}
		+\E_{a\sim \pi'(\cdot|s)} \sbr{
			\widehat Q_h^\pi(s, a) 
			- \l_h(s, a) - \P_h \widehat V_{h+1}^{\pi}(s, a)
		}
		\\
		&\quad + \E_{s' \sim \P_h(\cdot|s, a), a \sim \pi'(\cdot|s)} \sbr{
			\widehat V_{h+1}^{\pi}(s')
			- V_{h+1}^{\pi'}(s')
		}.
	\end{align*}
	Unrolling the above relation, the result follows.
\end{proof}

The next lemma is standard, for proof see e.g., \citet{hazan2016introduction,lattimore2020bandit}.
\begin{lemma}[\textbf{Entropy regularized OMD}]
\label{lem:omd}
	Let $\eta >0$, and $g_k \in \R^n$, $x_k \in \Delta(n)$ be a sequence of vectors such that for all $a$, $x_1(a) = 1/n$, for all $k\in [K], a\in [n]$, $\eta g_k(a) \geq -1$ and
	\begin{align*}
		x_{k+1}(a) &= \frac{x_k(a)e^{-\eta g_t(a)}}{\sum_{a'\in [n]} x_k(a')e^{-\eta g_k(a')}}
        .
	\end{align*}
	Then,
	\begin{align*}
		\max_{x\in \Delta_n} \cb{ \sum_{k=1}^K \ab{g_k, x_k - x} }
		\leq 
		\frac{\log n}{\eta } 
		+ \eta \sum_{k=1}^K \sum_{i=1}^n x_k(i)g_k(i)^2
		.
	\end{align*}
\end{lemma}
The next lemma establishes a regret bound for OMD with blocking, and follows from standard arguments. We provide a proof for completeness.
\begin{lemma}[\textbf{Entropy regularized OMD with blocking}]
\label{lem:blocking_omd}
Let $K \in \mathbb Z_+, \tau \leq K, J=\lceil K/\tau \rceil$, and set $T_j \eqq \cb{\tau(j-1) + 1, \ldots, \tau j}$ for all $j\in [J]$.
Assume
$\eta >0$, let $g_k \in \R^n$ be a sequence of vectors such that $\forall a, k; \eta g_k(a) \geq -1$ , and set
\begin{align*}
    g_{(j)} &= \frac1\tau \sum_{k\in T_j} g_k \; \forall j \in [J]
    \\
    x_{(j+1)}(a) &= 
          \frac{x_{(j)}(a)e^{-\eta g_{(j)}(a)}}{\sum_{a'\in [n]} x_{(j)}(a')e^{-\eta g_{(j)}(a')}}.
\end{align*}
Then if $x_k \in \Delta(n)$ are such that
$x_k = x_{(j)} \text{ for all } k \in T_j, j\in [J]$ we have
	\begin{align*}
		\max_{x\in \Delta_n} \cb{ \sum_{k=1}^K \ab{g_k, x_k - x} }
		\leq 
		\frac{\tau \log n}{\eta } 
  + \tau \max_k{\norm{g_k}_\infty}
		+ \eta \sum_{k=1}^K \sum_{i=1}^n x_k(i)g_k(i)^2
		.
	\end{align*}
\end{lemma}
\begin{proof}
    By applying \cref{lem:omd} on $g_{(j)}, x_{(j)}$, we get
    \begin{align*}
		\sum_{j=1}^J \ab{g_{(j)}, x_{(j)} - x^\star} 
		\leq 
		\frac{\log n}{\eta } 
		+ \eta \sum_{j=1}^J \sum_{i=1}^n x_{(j)}(i)g_{(j)}(i)^2
		.
	\end{align*}
    In addition,
    \begin{align*}
        \sum_{j=1}^J \ab{g_{(j)}, x_{(j)} - x^\star} 
        = \sum_{j=1}^J \ab{\frac{1}{|T_j|}\sum_{k\in T_j} g_k, x_{(j)} - x^\star} 
        =
        \sum_{j=1}^J \frac{1}{|T_j|}\sum_{k\in T_j} \ab{ g_k, x_{k} - x^\star} 
        \geq \frac1\tau 
        \sum_{k=1}^K \ab{ g_k, x_{k} - x^\star} 
    \end{align*}
    Further, by Jensen's inequality,
    \begin{align*}
        g_{(j)}(i)^2 = \br{\frac1{|T_j|} \sum_{k\in T_j} g_k(i)}^2
        = \frac{1}{|T_j|^2} \br{\sum_{k\in T_j} g_k(i)}^2
        \leq \frac{1}{|T_j|} \sum_{k\in T_j} g_k(i)^2,
    \end{align*}
    thus
    \begin{align*}
		\frac1\tau 
        \sum_{k=1}^K \ab{ g_k, x_{k} - x^\star} 
		\leq 
		\frac{\log n}{\eta } 
		+ \frac{\eta}{\tau} \sum_{k=1}^{K'} 
                \sum_{i=1}^n x_{k}(i)g_{k}(i)^2
            + \frac{\eta}{|T_J|} \sum_{k\in T_J}\sum_{i=1}^n x_{k}(i)g_{k}(i)^2
            ,
	\end{align*}
    where $K' = \max\cb{k\in T_{J-1}}$.
    Finally,
    \begin{align*}
        \sum_{k=1}^K \ab{ g_k, x_{k} - x^\star} 
		&\leq 
		\frac{\tau \log n}{\eta } 
		+ \eta \sum_{k=1}^{K'} 
                \sum_{i=1}^n x_{k}(i)g_{k}(i)^2
            + \frac{\tau \eta}{|T_J|} \sum_{k\in T_J}\sum_{i=1}^n x_{k}(i)g_{k}(i)^2
            ,
            \\
            &\leq 
		\frac{\tau \log n}{\eta } 
		+ \eta \sum_{k=1}^{K'} 
                \sum_{i=1}^n x_{k}(i)g_{k}(i)^2
            + \frac{\tau }{|T_J|} \sum_{k\in T_J}\sum_{i=1}^n x_{k}(i)g_{k}(i)
            \\
            &\leq 
		\frac{\tau \log n}{\eta } 
		+ \eta \sum_{k=1}^{K} 
                \sum_{i=1}^n x_{k}(i)g_{k}(i)^2
            + \tau \max_k{\norm{g_k}_\infty},
	\end{align*}
 which concludes the proof.
\end{proof}

\end{document}